\documentclass[runningheads]{llncs}

\usepackage[utf8]{inputenc}
\usepackage{graphicx}
\usepackage{color}
\usepackage{amsmath}
\usepackage{amssymb}
\usepackage{enumerate}
\usepackage{subfigure}
\usepackage[algo2e,inoutnumbered,linesnumbered,algoruled,vlined]{algorithm2e}
\usepackage{algpseudocode}
\usepackage{booktabs}
\usepackage{tabularx}
\usepackage{array}
\usepackage{bm}
\usepackage{enumitem}
\usepackage{xcolor}
\usepackage{wrapfig}
\usepackage{cutwin}
\usepackage[symbol]{footmisc}

\usepackage[width=122mm,left=12mm,paperwidth=146mm,height=193mm,top=12mm,paperheight=217mm]{geometry}

\usepackage{comment}
\usepackage{hyperref}
\usepackage{url}
\makeatother

\usepackage{amsmath,amsfonts,bm}

\def\eqref#1{equation~\ref{#1}}

\def\1{\bm{1}}

\newcommand{\test}{\mathcal{D_{\mathrm{test}}}}

\def\rvo{{\mathbf{o}}}
\def\rvp{{\mathbf{p}}}

\def\rvv{{\mathbf{v}}}
\def\rvw{{\mathbf{w}}}

\def\rmC{{\mathbf{C}}}

\DeclareMathAlphabet{\mathsfit}{\encodingdefault}{\sfdefault}{m}{sl}
\SetMathAlphabet{\mathsfit}{bold}{\encodingdefault}{\sfdefault}{bx}{n}

\def\gQ{{\mathcal{Q}}}

\newcommand{\R}{\mathbb{R}}

\DeclareMathOperator*{\argmax}{arg\,max}
\DeclareMathOperator*{\argmin}{arg\,min}

\newcommand{\x}{\mathbf{x}}

\newcommand{\Ag}{\mathbf{A}}
\newcommand{\Cg}{\mathbf{c}}

\newcommand{\B}{\mathcal{B}}
\newcommand{\Id}{\mathbf{I}}
\newcommand{\Man}{\mathcal{M}}
\newcommand{\T}{\mathcal{T}}

\newcommand{\zero}{\mathbf{0}}

\newcommand{\X}{\mathbf{X}}

\newcommand{\w}{\mathbf{w}}
\newcommand{\g}{\mathbf{g}}
\newcommand{\G}{\mathbf{G}}
\newcommand{\q}{\mathbf{q}}
\newcommand{\Q}{\mathbf{Q}}
\newcommand{\qc}{\mathbf{\bar{q}}}
\newcommand{\pt}{\mathbf{p}}
\newcommand{\tb}{\mathbf{t}}

\newcommand{\Set}{\mathbf{S}}
\newcommand{\Mean}{\mathcal{A}}
\newcommand{\QH}{{\mathbb{H}_1}} 
\newcommand{\nrml}{\mathbf{n}}
\newcommand{\M}{\mathbf{M}}

\newcommand{\Xo}{{\cal X}}
\newcommand{\Go}{{\cal G}}
\newcommand{\Lo}{{\cal L}}

\newcommand{\y}{\mathbf{y}}

\newcommand{\Y}{{\mathbf{Y}}}
\newcommand{\D}{ {\cal D} }

\newcommand{\Sp}{\mathbb{S}}
\newcommand{\Scal}{\mathcal{S}}

\newtheorem{thm}{Theorem}

\newtheorem{dfn}{Definition}

\newcommand{\etal}{\textit{et al}. }

\newcommand{\eg}{\textit{e}.\textit{g}. }

\usepackage{cleveref}

\Crefname{assumption}{\textbf{H}\hspace{-3pt}}{\textbf{H}\hspace{-3pt}}
\crefname{algorithm}{\text{Alg.}}{\text{Alg.}}
\crefname{assumption}{\textbf{H}}{\textbf{H}}
\crefname{equation}{\text{Eq}}{\text{Eq}}
\crefname{definition}{\text{Dfn.}}{\text{Dfn.}}
\crefname{lemma}{\text{Lemma}}{\text{Lemma}}
\crefname{remark}{\text{Remark}}{\text{Remark}}
\crefname{dfn}{\text{Dfn.}}{\text{Dfn.}}
\crefname{thm}{\text{Thm.}}{\text{Thm.}}
\crefname{tab}{\text{Tab.}}{\text{Tab.}}
\crefname{fig}{\text{Fig.}}{\text{Fig.}}
\crefname{table}{\text{Tab.}}{\text{Tab.}}
\crefname{figure}{\text{Fig.}}{\text{Fig.}}
\crefname{section}{\text{Sec.}}{\text{Sec.}}

\newcommand{\insertimageStar}[5]{ 
\begin{figure*}[#5]
\centering
\includegraphics[width=#1\linewidth, clip=true]{figures/#2}
\caption{#3}
\label{#4}
\end{figure*}
}

\begin{document}
\pagestyle{headings}
\mainmatter
\def\ECCVSubNumber{267}  

\title{Quaternion Equivariant Capsule Networks \\for 3D Point Clouds}

\titlerunning{QE-Networks}
%
\author{Yongheng Zhao\inst{1,3, *}\and
Tolga Birdal\inst{2, *}\and
Jan Eric Lenssen\inst{4}\and
Emanuele Menegatti\inst{1}\and
Leonidas Guibas\inst{2}\and
Federico Tombari\inst{3,5}}
\authorrunning{Y. Zhao et al.}
%
\institute{$^\text{1 } $ University of Padova \qquad $^\text{2 } $ Stanford University \qquad $^\text{3 } $  TU Munich \newline $^\text{4 } $  TU Dortmund \qquad $^\text{5 } $  Google}

\maketitle

\begin{abstract}
We present a 3D capsule module for processing point clouds that is equivariant to 3D rotations and translations, as well as invariant to permutations of the input points. The operator receives a sparse set of local reference frames, computed from an input point cloud and establishes end-to-end transformation equivariance through a novel dynamic routing procedure on quaternions. Further, we theoretically connect dynamic routing between capsules to the well-known Weiszfeld algorithm, a scheme for solving \emph{iterative re-weighted least squares} (IRLS) problems with provable convergence properties. It is shown that such group dynamic routing can be interpreted as robust IRLS rotation averaging on capsule votes, where information is routed based on the final inlier scores. Based on our operator, we build a capsule network that disentangles geometry from pose, paving the way for more informative descriptors and a structured latent space. Our architecture allows joint object classification and orientation estimation without explicit supervision of rotations. We validate our algorithm empirically on common benchmark datasets. We release our sources under: \href{tolgabirdal.github.io/qecnetworks}{https://tolgabirdal.github.io/qecnetworks}.\footnote[1]{First two authors contributed equally to this work.}
\end{abstract}
\section{Introduction}
\label{sec:intro}

It is now well understood that in order to learn a compact and informative representation of the input data, one needs to respect the symmetries in the problem domain~\cite{cohen2019gauge,weiler20183d}. Arguably, one of the primary reasons for the success of 2D convolutional neural networks (CNN) is the \textit{translation-invariance} of the 2D convolution acting on the image grid~\cite{giles1987learning,kondor2018clebsch}. Recent trends aim to transfer this success into the 3D domain in order to support many applications such as shape retrieval, shape manipulation, pose estimation, 3D object modeling and detection, etc. There, the data is naturally represented as sets of 3D points~\cite{qi2017pointnet,qi2017pointnet++}. Unfortunately, an extension of CNN architectures to 3D point clouds is non-trivial due to two reasons: 1) point clouds are irregular and unorganized, 2) the group of transformations that we are interested in is more complex as 3D data is often observed under arbitrary non-commutative $SO(3)$ rotations. As a result, learning appropriate embeddings requires 3D point-networks to be \textit{equivariant} to these transformations, while also being invariant to point permutations.
\insertimageStar{1}{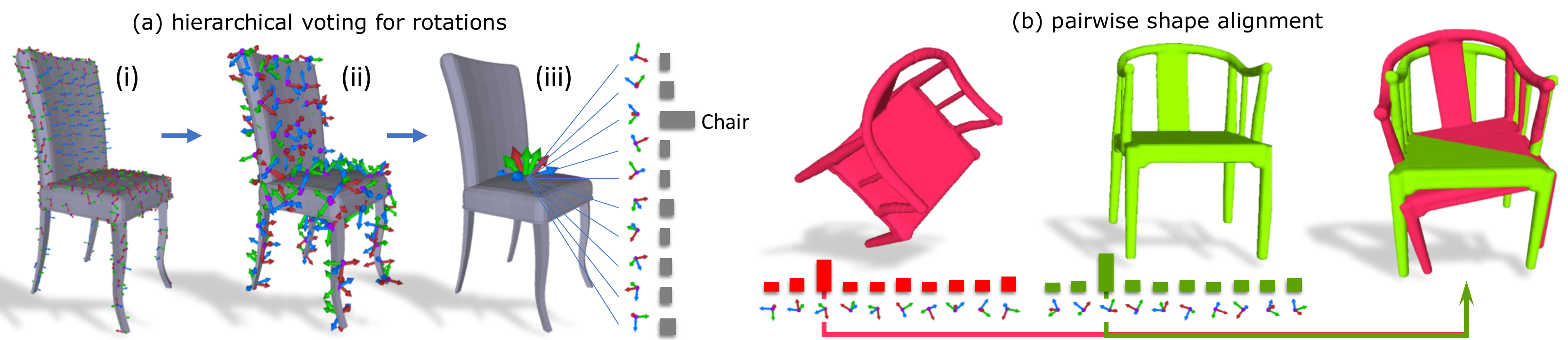}{(\textbf{a}) Our network operates on local reference frames (LRF) of an input point cloud (\textbf{i}). A hierarchy of quaternion equivariant capsule modules (QEC) then pools the LRFs to a set of latent capsules (\textbf{ii}, \textbf{iii}) disentangling the activations from poses. We can use activations in classification and the capsule (quaternion) with the highest activation in absolute (canonical) pose estimation without needing the supervision of rotations. (\textbf{b}) Our siamese variant can also solve for the relative object pose by aligning the capsules of two shapes with different point samplings. Our network directly consumes point sets and LRFs. Meshes are included only to ease understanding.}{fig:LRFs}{t!}

In order to fill this gap, we present a quaternion equivariant point capsule network that is suitable for processing point clouds and is equivariant to $SO(3)$ rotations, compactly parameterized by quaternions, while also preserving translation and permutation invariance. Inspired by the local group equivariance~\cite{lenssen2018group,cohen2019gauge}, we efficiently cover $SO(3)$ by restricting ourselves to a sparse set of local reference frames (LRFs) that collectively determine the object orientation. The proposed \emph{quaternion equivariant capsule (QEC) module} deduces equivariant latent representations by robustly combining those LRFs using the proposed \textit{Weiszfeld dynamic routing} with inlier scores as activations, so as to route information from one layer to the next. Hence, our latent features specify to local orientations and activations, disentangling orientation from evidence of object existence. Such explicit and factored storage of 3D information is unique to our work and allows us to perform rotation estimation jointly with object classification. Our final architecture is a hierarchy of QEC modules, where LRFs are routed from lower level to higher level capsules as shown in~\cref{fig:LRFs}. We use classification error as the only training cue and adapt a Siamese version for regression of the relative rotations.
We neither explicitly supervise the network with pose annotations nor train by augmenting rotations.
In summary, our contributions are:
\begin{enumerate}[noitemsep]
    \item We propose a novel, fully $SO(3)$-equivariant capsule module that produces invariant latent representations while explicitly decoupling the orientation into capsules. Notably, equivariance results have not been previously achieved for $SO(3)$ capsule networks. 
    \item We connect dynamic routing between capsules~\cite{sabour2017dynamic} and generalized Weiszfeld iterations~\cite{aftab2015}. Based on this connection, we theoretically argue for the convergence of the included rotation estimation on votes and extend our understanding of dynamic routing approaches.
    \item We propose a capsule network that is tailored for simultaneous classification and orientation estimation of 3D point clouds. We experimentally demonstrate the capabilities of our network on classification and orientation estimation on ModelNet10 and ModelNet40 3D shape data.
\end{enumerate}

\section{Related Work}
\label{sec:extendedrelated}
\paragraph{\textbf{Deep learning on point sets.}}
The capability to process raw, unordered point clouds within a neural network is introduced by the prosperous PointNet~\cite{qi2017pointnet} thanks to the point-wise convolutions and the permutation invariant pooling functions. Many works have extended PointNet primarily to increase the local receptive field size~\cite{qi2017pointnet++,li2018pointcnn,shen2018mining,dgcnn}. Point-clouds are generally thought of as sets. This makes any permutation-invariant network that can operate on sets an amenable choice for processing points~\cite{Zaheer2017,rezatofighi2017deepsetnet}. Unfortunately, common neural network operators in this category are solely equivariant to permutations and translations but to no other groups. 
\paragraph{\textbf{Equivariance in neural networks.}} Early attempts to achieve invariant data representations usually involved data augmentation techniques to accomplish tolerance to input transformations~\cite{maturana2015voxnet,qi2016volumetric,qi2017pointnet}. Motivated by the difficulty associated with augmentation efforts and acknowledging the importance of theoretically equivariant or invariant representations, the recent years have witnessed a leap in theory and practice of equivariant neural networks~\cite{bao2019equivariant,kondor2018generalization}. 

While laying out the fundamentals of the group convolution, G-CNNs~\cite{cohen2016group} guaranteed equivariance with respect to finite symmetry groups. Similarly, Steerable CNNs~\cite{cohen2016steerable} and its extension to 3D voxels~\cite{Worrall_2018_ECCV} considered discrete symmetries only. Other works opted for designing filters as a linear combination of harmonic basis functions, leading to frequency domain filters~\cite{Worrall_2017_CVPR,Weiler_2018_CVPR}. Apart from suffering from the dense coverage of the group using group convolution, filters living in the frequency space are less interpretable and less expressive than their spatial counterparts, as the basis does not span the full space of spatial filters.


Achieving equivariance in 3D is possible by simply generalizing the ideas of the 2D domain to 3D by voxelizing 3D data. However, methods using dense grids~\cite{chakraborty2018h,cohen2016steerable} suffer from increased storage costs, eventually rendering the implementations infeasible. An extensive line of work generalizes the harmonic basis filters to $SO(3)$ by using \eg, a spherical harmonic basis instead of circular harmonics~\cite{cohen2018spherical,esteves2018learning,CruzMota2012spherical}. In addition to the same downsides as their 2D counterparts, these approaches have in common that they require their input to be projected to the unit sphere~\cite{jiang2018spherical}, which poses additional problems for unstructured point clouds. A related line of research are methods which define a regular structure on the sphere to propose equivariant convolution operators~\cite{liu2018deep,Boomsma2017spherical}.



To learn a rotation equivariant representation of a 3D shape, one can either act on the input data or on the network. In the former case, one either presents augmented data to the network~\cite{qi2017pointnet,maturana2015voxnet} or ensures rotation-invariance in the input~\cite{deng2018ppf,deng2018ppfnet,khoury2017learning}. In the latter case one can enforce equivariance in the bottleneck so as to achieve an invariant latent representation of the input~\cite{mehr2018manifold,thomas2018tensor,spezialetti2019learning}. Further, equivariant networks for discrete sets of views~\cite{Esteves2019multiview} and cross-domain views~\cite{esteves2019cross} have been proposed. Here, we aim for a different way of embedding equivariance in the network by means of an explicit latent rotation parametrization in addition to the invariant feature. 

\paragraph{\textbf{Vector field networks}} \cite{Marcos_2017_ICCV} followed by the 3D \textit{Tensor Field Networks} (TFN)~\cite{thomas2018tensor} are closest to our work. Based upon a geometric algebra framework, the authors did achieve localized filters that are equivariant to rotations, translations and permutations. Moreover, they are able to cover the continuous groups. However, TFN are designed for physics applications, are memory consuming and a typical implementation is neither likely to handle the datasets we consider nor can provide orientations in an explicit manner.

\paragraph{\textbf{Capsule networks.}} The idea of capsule networks was first mentioned by Hinton~\etal~\cite{hinton2011transforming}, before Sabour~\etal~\cite{sabour2017dynamic} proposed the \textit{dynamic routing by agreement}, which started the recent line of work investigating the topic. Since then, routing by agreement has been connected to several well-known concepts, e.g. the EM algorithm~\cite{sabour2018matrix}, clustering with KL divergence regularization~\cite{wang2018an} and equivariance~\cite{lenssen2018group}.  They have been extended to autoencoders~\cite{kosiorek2019stacked} and GANs \cite{Jaiswal2019capsules}.
Further, capsule networks have been applied for specific kinds of input data, e.g. graphs~\cite{xinyi2018capsule}, 3D point clouds \cite{zhao20193d,srivastava2019geometric} or medical images~\cite{Afshar2018capsule}.
\section{Preliminaries and Technical Background}
We now provide the necessary background required for the grasp of the equivariance of point clouds under the action of quaternions.
\subsection{Equivariance}
\begin{dfn}[Equivariant Map]
\label{dfn:equiv}
For a $\Go$-space acting on $\Xo$, the map $\Phi : \Go \times \Xo \mapsto \Xo$ is said to be \textit{equivariant} if its domain and co-domain are acted on by the same symmetry group~\cite{cohen2016group,cohen2018general}:
\begin{align}
    \Phi (\g_1 \circ \x) = \g_2 \circ \Phi (\x)
\end{align}
where $\g_1\in \Go$ and $\g_2\in \Go$. Equivalently $
    \Phi (T(\g_1)\,\x) = T(\g_2)\,\Phi (\x)$, 
where $T(\cdot)$ is a linear representation of the group $\Go$. Note that $T(\cdot)$ does not have to commute. It suffices for $T(\cdot)$ to be a homomorphism: $T (\g_1 \circ \g_2) = T (\g_1) \circ T(\g_2)$. In this paper we use a stricter form of equivariance and consider $\g_2=\g_1$.
\end{dfn}
\begin{dfn}[Equivariant Network]
An architecture or network is said to be equivariant if all of its layers are equivariant maps. Due to the transitivity of the equivariance, stacking up equivariant layers will result in globally equivariant networks \eg, rotating the input will produce output vectors which are transformed by the same rotation~\cite{lenssen2018group,kondor2018generalization}.
\end{dfn}
\subsection{The Quaternion Group $\QH$}
The choice of 4-vector quaternions as representation for $SO(3)$ has multiple motivations: (1) All 3-vector formulations suffer from infinitely many singularities as angle goes to $0$, whereas quaternions avoid those, (2) 3-vectors also suffer from infinitely many redundancies (the norm can grow indefinitely). Quaternions have a single redundancy: $q=-q$ that is in practice easy to enforce~\cite{birdal2020synchronizing}, (3) Computing the actual ‘manifold mean’ on the Lie algebra requires iterative techniques with subsequent updates on the tangent space. Such iterations are computationally and numerically harmful for a differentiable GPU implementation.
\begin{dfn}[Quaternion]\label{dfn:quaternion}
A \emph{quaternion} $\q$ is an element of Hamilton algebra $\QH$, extending the complex numbers with three imaginary units $\textbf{i}$, $\textbf{j}$, $\textbf{k}$ in the form:
$\q
		= q_1 \textbf{1} + q_2 \textbf{i} + q_3 \textbf{j} + q_4 \textbf{k}
    = \left(q_1, q_2, q_3, q_4\right)^{\text{T}}$,
with $\left(q_1, q_2, q_3, q_4\right)^{\text{T}} \in \mathbb{R}^4$ and
$\textbf{i}^2 = \textbf{j}^2 = \textbf{k}^2 = \textbf{i}\textbf{j}\textbf{k} = - \textbf{1}$. $q_1 \in \mathbb{R}$ denotes the scalar part and $\textbf{v} = \left(q_2, q_3, q_ 4\right)^{\text{T}} \in \mathbb{R}^3$, the vector part.
The \emph{conjugate} $\bar{\q}$ of the quaternion $\q$ is given by $
\qc := q_1 - q_2 \textbf{i} - q_3 \textbf{j} - q_4 \textbf{k}$.
A \emph{unit quaternion} $\q \in \mathbb{H}_1$ with $1 \stackrel{\text{!}}{=} \left\|\q\right\|
	:= \q \cdot \qc$ and $\q^{-1}= \qc$,
gives a compact and numerically stable parametrization to represent orientation of objects on the unit sphere $\Scal^3$, avoiding gimbal lock and singularities~\cite{busam2016_iccvw}.  Identifying antipodal points $\q$ and $-\q$ with the same element, the unit quaternions form a double covering group of $SO\left(3\right)$. 
$\QH$ is closed under the non-commutative multiplication or the Hamilton product:
\begin{align}
    (\pt\in \QH) \circ (\textbf{r}\in \QH) =
[{p}_1{r}_1-\mathbf{v}_p\cdot \mathbf{v}_r\, ;\,{p}_1\mathbf{v}_r+{r}_1 \mathbf{v}_p+\mathbf{v}_p \times \mathbf{v}_r].
\end{align}
\end{dfn}


\begin{dfn}[Linear Representation of $\QH$]
We follow~\cite{birdal2018bayesian} and use the \textit{parallelizable} nature of unit quaternions ($d\in\{1,2,4,8\}$ where $d$ is the dimension of the ambient space) to define $T: \QH \mapsto \R^{4 \times 4} $ as:
\[
\label{eq:Qlinear}
\mathbf{T}(\q) \triangleq 
\left[ \begin{array}{@{}cccc@{}}
q_1 & -q_2 				& -q_3 				&  {-}q_4 \\
q_2 & \phantom{-}q_1 	& {-}q_4 	&  \phantom{-}q_3\\
q_3 & \phantom{-}q_4 				& \phantom{-}q_1 	&  -q_2\\
q_4 & {-}q_3 	& \phantom{-}q_2				&   \phantom{-}q_1
\end{array}
\right] \hspace{-3pt}.
\]
To be concise we will use capital letters to refer to the matrix representation of quaternions~\eg $\Q \equiv T(\q)$, $\G \equiv T(\g)$. Note that $T(\cdot)$, the injective homomorphism to the orthonormal matrix ring, by construction satisfies the condition in~\cref{dfn:equiv}~\cite{Steenrod1951}: $\det(\Q) = 1,\Q^\top = \Q^{-1}, \|\Q\|=\|\Q_{i,:}\|=\|\Q_{:,i}\|=1$ and $\Q-q_1\Id$ is skew symmetric: $\Q+\Q^\top=2q_1\Id$. It is easy to verify these properties. $T$ linearizes the Hamilton product or the group composition: $\g \circ \q \triangleq T(\g) \q \triangleq \G \q$.
\end{dfn}

\subsection{3D Point Clouds}
\begin{dfn}[Point Cloud]
We define a 3D surface to be a differentiable 2-manifold embedded in the ambient 3D Euclidean space: $\Man^2\in\R^3$ and a point cloud to be a discrete subset sampled on $\Man^2$: $\X\in\{\x_i\in\Man^2\cap\R^3\}$.
\end{dfn}
\begin{dfn}[Local Geometry]
For a smooth point cloud $\{\x_i\}\in\Man^2 \subset \R^{N\times 3}$, a \emph{local reference frame (LRF)} is defined as an ordered basis of the tangent space at $\x$, $\T_\x\Man$, consisting of orthonormal vectors: $\Lo(\x)=[\bm{\partial}_1, \bm{\partial}_2, \bm{\partial}_3\equiv\bm{\partial}_1\times \bm{\partial}_2]$.
Usually the first component is defined to be the surface normal $\bm{\partial}_1\triangleq\nrml\in\Scal^2 : \|\nrml\|=1$ and the second one is picked according to a heuristic. 
\end{dfn}
Note that recent trends, \eg~as in Cohen~\etal~\cite{cohen2019gauge}, acknowledge the ambiguity and either employ a \textit{gauge} (tangent frame) equivariant design or propagate the determination of a certain direction until the last layer~\cite{poulenard2018multi}. Here, we will assume that $\bm{\partial}_2$ can be uniquely and repeatably computed, a reasonable assumption for the point sets we consider~\cite{petrelli2011repeatability}. For the cases where this does not hold, we will rely on the robustness of the iterative routing procedures in our network. We will explain our method of choice in~\cref{sec:exp} and visualize LRFs of an airplane object in~\cref{fig:LRFs}.

\section{$SO(3)$-Equivariant Dynamic Routing}
\label{sec:method}
Disentangling orientation from representations requires guaranteed equivariances and invariances. Yet, the original capsule networks of Sabour~\etal~\cite{sabour2017dynamic} cannot achieve equivariance to general groups. To this end, Lenssen~\etal~\cite{lenssen2018group} proposed a dynamic routing procedure that guarantees equivariance and invariance under $SO(2)$ actions, by applying a manifold-mean and the geodesic distance as routing operators.
We will extend this idea to the non-abelian $SO(3)$ and design capsule networks that sparsely operate on a set of LRFs computed via~\cite{petrelli2012repeatable} on local neighborhoods of points. The $SO(3)$ elements are paremeterized by quaternions similar to~\cite{zhang2020quaternion}.
In the following, we begin by introducing our novel equivariant dynamic routing procedure, the main building block of our architecture. We show the connection to the well known Weiszfeld algorithm, broadening the understanding of dynamic routing by embedding it into traditional computer vision methodology. Then, we present an example of how to stack those layers via a simple aggregation, resulting in an $SO(3)$-equivariant 3D capsule network that yields invariant representations (or activations) as well as equivariant orientations (latent capsules). 

\subsection{Equivariant Quaternion Mean}
\label{sec:quatLayer}
To construct equivariant layers on the group of rotations, we are required to define a left-equivariant averaging operator $\Mean$ that is invariant under permutations of the group elements, as well as a distance metric $\delta$ that remains unchanged under the action of the group \cite{lenssen2018group}. For these, we make the following choices:
\begin{dfn}[Geodesic Distance]\label{dfn:delta}
The Riemannian (geodesic) distance on the manifold of rotations lead to the following geodesic distance $\delta(\cdot)\equiv d_{\text{quat}}(\cdot)$:
\begin{align}
\label{eq:qdist}
d(\q_1,\q_2)\equiv d_{\text{quat}}(\q_1,\q_2)=2\cos^{-1}&(|\langle \q_1, \q_2\rangle|)
\end{align}
\end{dfn}

\begin{dfn}[Quaternion Mean $\bm{\mu}(\cdot)$]
\label{dfn:mean}
For a set of $Q$ rotations $\Set=\{\q_i\}$ and associated weights $\w = \{w_i\}$, the weighted mean operator $\Mean(\Set, \w): \QH^n \times \R^n \mapsto \QH^n$ is defined through the following maximization procedure~\cite{markley2007averaging}:
\begin{align}
\label{eq:mean}
\bar{\q} = \argmax_{\q\in \Sp^3} \q^\top\M\q
\end{align}
where $\M\in\R^{4\times 4}$ is defined as: $\M \triangleq \sum\limits_{i=1}^{Q} w_i \q_i\q_i^\top$.
\end{dfn}
The average quaternion $\bar{\q}$ is the eigenvector of $\M$ corresponding to the maximum eigenvalue. This operation lends itself to both analytic~\cite{magnus1985differentiating} and automatic differentiation~\cite{laue2018computing}. The following properties allow $\Mean(\Set, \w)$ to be used to build an equivariant dynamic routing:

\begin{thm}
\label{thm:all} 
Quaternions, the employed mean $\Mean(\Set, \w)$ and geodesic distance $\delta(\cdot)$ enjoy the following properties:
\begin{enumerate}
    \item $\Mean(\g \circ \Set, \w)$ is left-equivariant: $\Mean(\g \circ \Set, \w) = \g \circ \Mean(\Set, \w)$.
    \item Operator $\Mean$ is invariant under permutations: \begin{equation}\Mean(\{\q_1,\dots,\q_Q\}, \w)=\Mean(\{\q_{\sigma(1)},\dots,\q_{\sigma(Q)}\}, \w_\sigma).
    \end{equation}
    \item The transformations $\g \in \QH$ preserve the geodesic distance $\delta(\cdot)$ given in~\cref{dfn:delta}.
\end{enumerate}
\end{thm}
\begin{proof}
The proofs are given in the supplementary material.
\end{proof}
We also note that the above mean is closed form, differentiable and can be computed in a batch-wise fashion. 
We are now ready to construct the \textit{dynamic routing} (DR) by agreement that is equivariant to $SO(3)$ actions, thanks to~\cref{thm:all}.

 \begin{algorithm2e} [t!]
 \DontPrintSemicolon
 \SetKwInOut{Input}{input}
 \SetKwInOut{Output}{output}
 \Input{Input points $\{\mathbf{x}_1,...,\mathbf{x}_K\} \in \mathbb{R}^{K \times 3}$, input capsules (LRFs) $\gQ = \{\q_1,\dots, \q_L\}\in \QH^L$, with $L= N^{c} \cdot K$, $N^{c}$ is the number of capsules per point, activations $\bm{\alpha} = (\alpha_1, \dots, \alpha_L)^T$, trainable transformations $\T=\{\tb_{i,j}\}_{i,j}\in\QH^{L\times M}$}
 \Output{Updated frames $\hat{\gQ} = \{\hat{\q}_1,\dots, \hat{\q}_M\}\in \QH^M$, updated activations $\hat{\bm{\alpha}} = (\hat{\alpha}_1, \dots, \hat{\alpha}_M)^T$}
 \For{All primary (input) capsules $i$}{
    \For{All latent (output) capsules $j$}{
    $\rvv_{i,j} \gets \q_i \circ \tb_{i,j}\,$ {\color{purple} \small \tcp{compute votes}}
    }
 }
 \For{All latent (output) capsules $j$}{
    $\hat{\q}_{j} \gets \Mean\big( \{\rvv_{1,j}\dots\rvv_{K,j}\}, \bm{\alpha} \big)\,\,${\color{purple} \small \tcp{initialize output capsules}}
    \For{$k$ iterations}{
        \For{All primary (input) capsules $i$}{
        $w_{i,j}\gets \alpha_{i} \cdot \text{sigmoid}\big(-\delta(\hat{\q}_j,\rvv_{i,j})\big)\,\,$ {\color{purple} \small \tcp{the current weight}}
    }
    $\hat{\q}_{j} \gets \Mean\big( \{\rvv_{1,j}\dots\rvv_{L,j}\}, \rvw_{:,j} \big)\,\,${\color{purple} \small \tcp{see~\cref{eq:mean}}}
    }
    $\hat{\alpha}_{j} \gets \text{sigmoid}\big(-\frac{1}{K}\sum\limits_{1}^{L} \delta(\hat{\q}_j,\rvv_{i,j}) \big)\,${\color{purple} \small \tcp{recompute activations}} 
    }
 \caption{Quaternion Equivariant Dynamic Routing}
 \label{algo:DR}
 \end{algorithm2e}

\subsection{Equivariant Weiszfeld Dynamic Routing}
Our routing procedure extends previous work~\cite{sabour2017dynamic,lenssen2018group} for quaternion valued input. The core idea is to \textit{route} from the \textit{primary capsules} that constitute the input LRF set to the \textit{latent capsules} by an iterative clustering of votes $\mathbf{v}_{i,j}$. At each step, we assign the weighted group mean of votes to the respective output capsules. The weights $w\gets \sigma(\x,\y)$ are inversely propotional to the distance between the vote quaternions and the new quaternion (cluster center). See~\cref{algo:DR} for details. In the following, we analyze our variant of routing as an interesting case of the affine, Riemannian Weiszfeld algorithm~\cite{aftab2015,aftab2014generalized}.
\begin{lemma}
\label{lem:weisz}
For $\sigma(\x,\y) = {\delta(\x,\y)^{q-2}}$ the equivariant routing procedure given in~\cref{algo:DR} is a variant of the affine subspace Wieszfeld algorithm~\cite{aftab2015,aftab2014generalized} that is a robust algorithm for computing the $L_q$ geometric median.
\end{lemma}
\begin{proof}[{Proof Sketch}]
The proof follows from the definition of Weiszfeld iteration~\cite{aftab2014generalized} and the mean and distance operators defined in~\cref{sec:quatLayer}. We first show that computing the weighted mean is equivalent to solving the normal equations in the iteratively reweighted least squares (IRLS) scheme~\cite{burrus2012iterative}. Then, the inner-most loop corresponds to the IRLS or Weiszfeld iterations. We provide the detailed proof in supplementary material.
\end{proof}

Note that, in practice one is quite free to choose the weighting function $\sigma(\cdot)$ as long as it is inversely proportional to the geodesic distance and concave~\cite{aftab2015convergence}. 
The original dynamic routing can also be formulated as a clustering procedure with a KL divergence regularization. This holistic view paves the way to better routing algorithms~\cite{wang2018an}. Our perspective is akin yet more geometric due to the group structure of the parameter space.
Thanks to the connection to Weiszfeld algorithm, the convergence behavior of our dynamic routing can be directly analyzed within the theoretical framework presented by~\cite{aftab2014generalized,aftab2015}.
\insertimageStar{1}{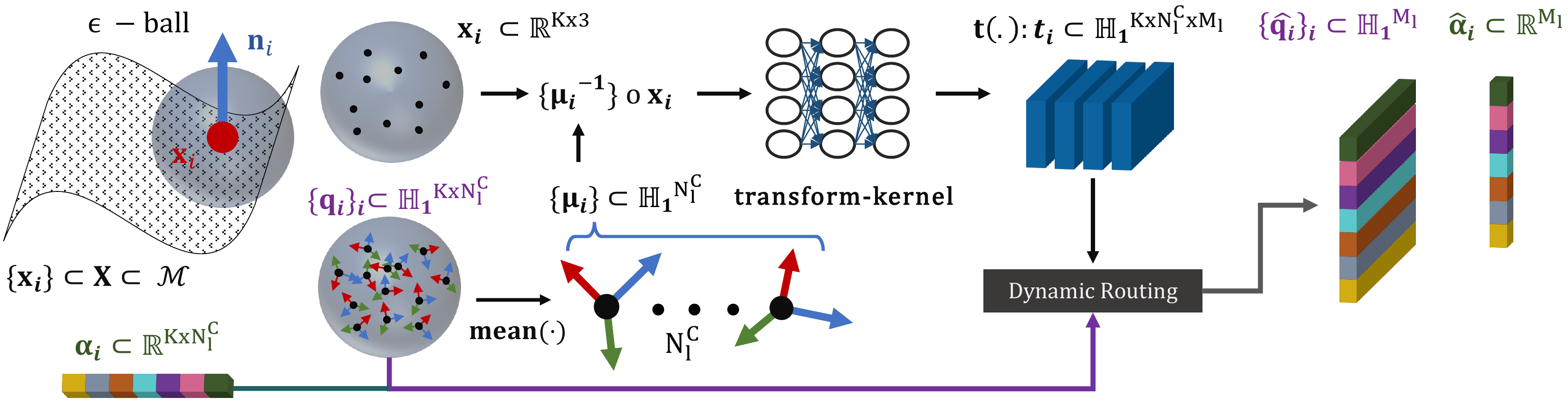}{Our \textbf{q}uaternion \textbf{e}quivariant \textbf{c}apsule (QEC) layer for processing local patches: Our input is a 3D point set $\mathbf{X}$ on which we query local neighborhoods $\{\x_i\}$ with precomputed LRFs $\{\mathbf{q}_i\}$. Essentially, we learn the parameters of a fully connected network that continuously maps the canonicalized local point set to transformations $\mathbf{t}_i$, which are used to compute hypotheses (votes) from input capsules. By a special dynamic routing procedure that uses the activations determined in a previous layer, we arrive at latent capsules that are composed of a set of orientations $\hat{\mathbf{q}}_i$ and new activations $\hat{\bm{\alpha}}_i$. Thanks to the decoupling of local reference frames, $\hat{\bm{\alpha}}_i$ is invariant and orientations $\hat{\mathbf{q}}_i$ are equivariant to input rotations. All the operations and hence the entire QE-network are equivariant achieving a guaranteed disentanglement of the rotation parameters. \textit{Hat symbol ($\hat{\q}$) refers to 'estimated'.}}{fig:qlayer}{t!}
\begin{thm}
Under mild assumptions provided in the appendix, the sequence of the DR-iterates generated by the inner-most loop almost surely converges to a critical point.
\end{thm}
\begin{proof}[Proof Sketch]
Proof, given in the appendix, is a direct consequence of~\cref{lem:weisz} and directly exploits the connection to the Weiszfeld algorithm.
\end{proof}
In summary, the provided theorems show that our dynamic routing by agreement is in fact a variant of robust IRLS rotation averaging on the predicted votes, where refined inlier scores for combinations of input/output capsules are used to route information from one layer to the next.


\section{Equivariant Capsule Network Architecture}
\label{sec:arch}
In the following, we describe how we leverage the novel dynamic routing algorithm to build a capsule network for point cloud processing that is equivariant under $SO(3)$ actions on the input.
The essential ingredient of our architecture, the \emph{quaternion equivariant capsule (QEC) module} that implements a capsule layer with dynamic routing, is  described in Sec. \ref{sec:qec_module}, before using it as building block in the full architecture, as described in Sec. \ref{sec:full_architecture}.
  \begin{algorithm2e}[t]
 \DontPrintSemicolon
 \SetKwInOut{Input}{input}
 \SetKwInOut{Output}{output}
 \Input{Input points of one patch $\{\mathbf{x}_1,...,\mathbf{x}_K\} \in \mathbb{R}^{K \times 3}$, input capsules (LRFs) $\gQ = \{\q_1,\dots, \q_L\}\in \QH^L$, with $L= N^{c} \cdot K$, $N^{c}$ is the number of capsules per point, activations $\bm{\alpha} = (\alpha_1, \dots, \alpha_L)^T$}
 \Output{Updated frames $\hat{\gQ} = \{\hat{\q}_1,\dots, \hat{\q}_M\}\in \QH^M$, updated activations $\hat{\bm{\alpha}} = (\hat{\alpha}_1, \dots, \hat{\alpha}_M)^T$}
 \For{Each input channel $n^{c}$ of all the primary capsules channels $N^{c}$}{
    $\mu(n^{c}) \gets \Mean( \gQ(n^{c}))$ {\color{purple} \footnotesize \tcp{Input quaternion average, see~\cref{eq:mean}}}
     \For{Each point $\mathbf{x}_i$ of this patch}{
         $\mathbf{x}_i' \gets {\mu(n^{c})}^{-1} \circ \mathbf{x}_i$ {\color{purple} \footnotesize \tcp{\footnotesize Rotate to a canonical orientation}}
     }
    }
    $\{\mathbf{x}_i'\} \in \mathbb{R}^{K \times N^{c} \times 3} ${\color{purple} \footnotesize \tcp{Points in multiple($N^{c}$) canonical frames}}
\For{Each point $\mathbf{x}_i'$ of this patch}{        
    $\tb \gets t(\mathbf{x}_i')$ {\color{purple} \footnotesize \tcp{\footnotesize Transform kernel, $t(\cdot): \mathbb{R}^{N^{c} \times 3} \rightarrow \mathbb{R}^{N^{c}\times M\times 4}$}}
}

 $\T\equiv\{\tb_i\}\in\QH^{K \times N_{i}^{c} \times M} \gets \{\tb\}\in\QH^{L\times M}$ \\
 $(\hat{\gQ}, \hat{\bm{\alpha}}) \gets \text{DynamicRouting} ( X, \gQ, \bm{\alpha}, \T)$ {\color{purple} \footnotesize \tcp{ \footnotesize See~\cref{algo:DR}}}
 \caption{Quaternion Equivariant Capsule Module}
 \label{algo:QE-N}
 \end{algorithm2e}

\subsection{QEC Module}
\label{sec:qec_module}
The main module of our architecture, the QEC module, is outlined in~\cref{fig:qlayer}. We also provide the corresponding pseudocode in ~\cref{algo:QE-N}.
\paragraph{\textbf{Input.}}
The input to the module is a local patch of points with coordinates $\mathbf{x}_i\subset\mathbb{R}^{K\times 3}$, rotations (LRFs) attached to these points, parametrized as quaternions $\mathbf{q}_i\subset\QH^{K\times N^{c}}$ and activations $\bm{\alpha}_i\subset \mathbb{R}^{K\times N^{c}}$. We also use $\mathbf{q}_i$ to denote the input capsules. $N^{c}$ is the number of input capsule channels per point and it is equal to the number of output capsules ($M$) from the last layer.

\paragraph{\textbf{Trainable transformations.}}
Recalling the original capsule networks of Sabour \etal~\cite{sabour2017dynamic}, the trainable transformations $\mathbf{t}$, which are applied to the input rotations to compute the votes, lie in a grid kernel in the 2D image domain.
Therefore, the procedure can learn to produce well-aligned votes if and only if the learned patterns in $\mathbf{t}$ match those in input capsule sets (agreement on evidence of object existence).
Since our input points in the local receptive field lie in continuous $\mathbb{R}^3$, training a discrete set of pose transformations $\mathbf{t}_{i,j}$ based on discrete local coordinates is not possible. Instead, we use a similar approach as Lenssen~\etal~\cite{lenssen2018group} and employ a continuous kernel $t(\cdot): \mathbb{R}^{N^{c}\times 3} \rightarrow \mathbb{R}^{M\times N^{c}\times 4}$ that is defined on the continuous $\mathbb{R}^{N^{c}\times 3}$, instead of only a discrete set of positions. The network is shared over all points to compute the transformations $\mathbf{t}_{i,j}= (t(\mathbf{x}_1'),...,t(\mathbf{x}_K'))_{i,j} \subset \mathbb{R}^{K\times M\times N^{c} \times 4}$, which are used to calculate the votes for dynamic routing with $\mathbf{v}_{i,j} = \mathbf{q}_{i} \circ \mathbf{t}_{i,j}$. The network $t(\cdot)$ consists of fully-connected layers that regresses the transformations, similar to common operators for continuous convolutions~\cite{Schutt2017Schnet,Wang_2018_CVPR,Fey_2018_CVPR}, just with quaternion output. The kernel is able to learn pose patterns in the 3D space, which align the resulting votes if certain pose sets are present. Note that $t(\cdot)$ predicts quaternions by unit-normalizing the regressed output: $\mathbf{t}_{i,j}\subset\QH^{K\times M\times N^{c}}$. Although Riemannian layers~\cite{becigneul2018riemannian} or spherical predictions~\cite{liao2019spherical} can improve the performance, the simple strategy works reasonably for our case. 

In order for the kernel to be invariant, it needs to be aligned using an equivariant initial orientation candidate~\cite{lenssen2018group}. Given points $\mathbf{x}_i$ and rotations $\mathbf{q}_i$, we compute the mean $\bm{\mu}_i$ in a channel-wise manner like that of the initial candidates: $\bm{\mu}_i\subset\QH^{N^{c}}$. These candidates are used to bring the kernels in canonical orientations by inversely rotating the input points: $\mathbf{x}_i' = ({\bm{\mu}_i}^{-1}\circ\mathbf{x}_i)\subset\mathbb{R}^{K\times N^{c}\times3}$.

\paragraph{\textbf{Computing the output.}}
After computing the votes, we utilize the input activation $\bm{\alpha}_i$ as initialization weights and iteratively refine the output capsule rotations (robust rotation estimation on votes) $\hat{\mathbf{q}_i}$ and activations $\hat{\bm{\alpha}_i}$ (final inlier scores) by our Weiszfeld routing by agreement as shown in~\cref{algo:DR}. 

\insertimageStar{1}{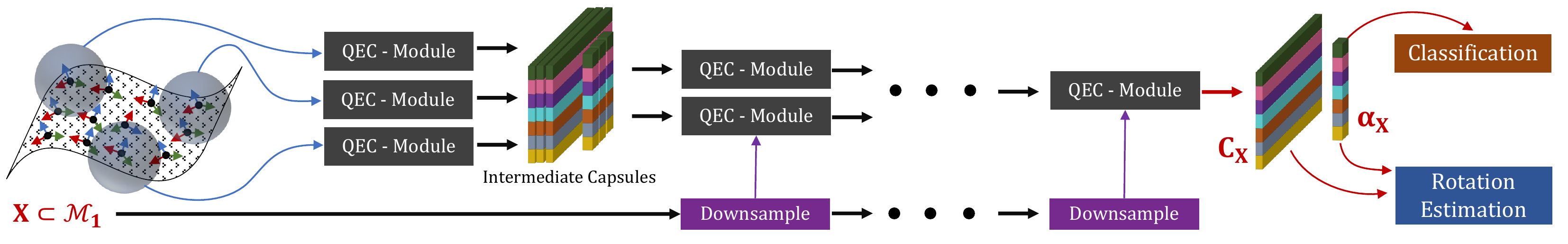}{Our entire capsule-network architecture. We hierarchically send all the local patches to our QEC-module as shown in~\cref{fig:qlayer}. At each level the points are pooled in order to increase the receptive field, gradually reducing the LRFs into a single capsule per class. We use classification and orientation estimation (in the siamese case) as supervision cues to train the transform-kernels $\tb(\cdot)$.}{fig:qnetwork}{t!}

\subsection{Network Architecture}
\label{sec:full_architecture}
For processing point clouds, we use multiple QEC modules in a hierarchical architecture as shown in~\cref{fig:qnetwork}. In the first layer, the input primary capsules are represented by LRFs computed with FLARE algorithm~\cite{petrelli2012repeatable}. Therefore, the number of input capsule channels $N^{c}$ in the first layer is equal to $1$ and activations are uniform. The output of a former layer is propagated to the input of the latter, creating the hierarchy.

In order to gradually increase the receptive field, we stack QEC modules creating a deep hierarchy, where each layer reduces the number of points and increases the receptive field. In our experiments, we use a two level architecture, which receives $N=64$ patches as input. We call the centers of these patches \textit{pooling centers} and compute them via a uniform farthest point sampling as in~\cite{birdal2017point}. Pooling centers serve as the positions of output capsules of the current layer. Each of those centers is linked to their immediate vicinity leading to $K=9$-star local connectivity from which serve as input to the first QEC module to compute rotations and activations of $64 \times 64 \times 4$ intermediate capsules. The second module connects those intermediate capsules to the output capsules, whose number corresponds to the number of classes. Specifically, for layer 1, we use $K=9,{N_l}^c=1, M_l=64$ and for layer 2, $K=64,{N_l}^c=64, M_l=C=40$. This way, the last QEC module receives only one input patch and pools all capsules into a single point with an estimated LRF. For further details, we refer to our source code, which we will make available online before publication and provide in the supplemental materials.


\section{Experimental Evaluations}
\label{sec:exp}
\begin{table}[t]    
  \centering
  \caption{Classification accuracy on ModelNet40 dataset~\cite{wu20153d} for different methods as well as ours. We also report the number of parameters optimized for each method. \textbf{X}/\textbf{Y} means that we train with \textbf{X} and test with \textbf{Y}.}
  \setlength{\tabcolsep}{2.0pt}
  \resizebox{\textwidth}{!}
  {
    \begin{tabular}{lcccccccccc}
    \toprule\toprule
    & PN & PN++ & DGCNN & KDTreeNet & Point2Seq & Sph.CNNs & PRIN  & PPF & Ours (Var.) & Ours \\
    \midrule
    \textbf{NR/NR} & 88.45 & 89.82 & \textbf{92.90} & 86.20  & {92.60} & -     & 80.13 & 70.16 & 85.27 & 74.43 \\
    \textbf{NR/AR} & 12.47 & 21.35 & 29.74 & 8.49  & 10.53 & 43.92 & 68.85 & 70.16 & 11.75 & \textbf{74.07} \\
    \midrule
    \textbf{\#Params} & 3.5M  & 1.5M & 2.8M & 3.6M  & 1.8M  & 0.5M  & 1.5M & 3.5M & 0.4M & \textbf{0.4M} \\
    \end{tabular}
    }
  \label{tab:classification}%
\end{table}%
\paragraph{\textbf{Implementation details.}}
We implement our network in PyTorch and use the ADAM optimizer~\cite{kingma2014adam} with a learning rate of $0.001$. Our point-transformation mapping network (transform-kernel) is implemented by two FC-layers composed of 64 hidden units. We set the initial activation of the input LRF to $1.0$. In each layer, we use 3 iterations of DR. For classification we use the spread loss~\cite{sabour2018matrix} and the rotation loss is identical to $\delta(\cdot)$. 

The first axis of the LRF is the surface normal computed by local plane fits~\cite{Hoppe1992}. We compute the second axis, $\bm{\partial}_2$, by FLARE~\cite{petrelli2012repeatable}, that uses the normalized projection of the point with the largest distance within the periphery of the support, onto the tangent plane of the center: $\bm{\partial}_2=\frac{\rvp_{\text{max}}-\rvp}{\|\rvp_{\text{max}}-\rvp\|}$. Using other choices such as SHOT~\cite{tombari2010unique} or GFrames~\cite{Melzi_2019_CVPR} is possible. We found FLARE to be sufficient for our experiments. Prior to all operations, we flip all the LRF quaternions such that they lie on the northern hemisphere : $\{\q_i\in \Sp^3 : q_i^w >0\}$.

\paragraph{\textbf{3D shape classification.}}
We use ModelNet40 dataset of~\cite{wu20153d,qi2017pointnet++} to assess our classification performance where each shape is composed of $10K$ points randomly sampled from the mesh surfaces of each shape~\cite{qi2017pointnet,qi2017pointnet++}. We use the official split with 9,843 shapes for training and 2,468 for testing. 
We assign the LRFs to a subset of the uniformly sampled points, $N=512$~\cite{birdal2017point}. 

During training, we do not augment the dataset with random rotations. All the shapes are trained with single orientation (well-aligned). We call this \textit{trained with NR}.
During testing, we randomly generate multiple arbitrary $SO(3)$ rotations for each shape and evaluate the average performance for all the rotations. This is called \textit{test with AR}.
This protocol is similar to~\cite{aoki2019pointnetlk}'s and is used both for our algorithms and for the baselines.
Our results are shown in~\cref{tab:classification} along with that of PointNet (PN)~\cite{qi2017pointnet}, PointNet++ (PN++)~\cite{qi2017pointnet}, DGCNN~\cite{wang2019dynamic}, KD-treeNet~\cite{li2018so}, Point2Seq~\cite{liu2019point2sequence}, Spherical CNNs~\cite{esteves2018learning}, PRIN~\cite{You2018prin} and the theoretically invariant PPF-FoldNet (PPF)~\cite{deng2018ppf}. We also present a version of our algorithm (\textit{Var}) that avoids the canonicalization within the QE-network. This is a non-equivariant network that we still train without data augmentation or orientation supervision. While this version gets comparable results to the state of the art for the NR/NR case, it cannot handle random $SO(3)$ variations (AR). Note that PPF uses the point-pair-feature~\cite{birdal2015point} encoding and hence creates invariant input representations. For the scenario of NR/AR, our equivariant version outperforms all the other methods, including equivariant spherical CNNs~\cite{esteves2018learning} by a significant gap of at least $5\%$ even when~\cite{esteves2018learning} exploits the 3D mesh. The object rotational symmetries in this dataset are responsible for a significant portion of the errors we make and we provide further details in supplementary material. It is worth mentioning that we also trained TFNs~\cite{thomas2018tensor} for that task, but their memory demand made it infeasible to scale to this application.

\begin{table}[t!]
  \centering
  \caption{Relative angular error (RAE) of rotation estimation in different categories of ModelNet10. Right side of the table denotes the objects with rotational symmetry, which we include for completeness. PCA-S refers to running PCA only on a resampled instance, while PCA-SR applies both rotations and resampling.}
  \setlength{\tabcolsep}{2.5pt}
  \resizebox{\textwidth}{!}
  {
    \begin{tabular}{lccccccc|ccccc}
    \toprule\toprule
    Method & Avg.   & No\_Sym & Chair & Bed   & Sofa  & Toilet  & Monitor & Table & Desk  & Dresser & NS & \multicolumn{1}{l}{Bathtub} \\
    \midrule
    Mean LRF & 0.41  & 0.35  & 0.32  & 0.36  & 0.34  & 0.41  & 0.34  & 0.45  & 0.60  & 0.50  & 0.46  & 0.32 \\
    PCA-S & 0.40  & 0.42  & 0.60  & 0.53  & 0.46  & 0.32  & 0.12  & 0.47  & \textbf{0.23}  & \textbf{0.33} & 0.43  & 0.55 \\
    PCA-SR & 0.67  & 0.67  & 0.69  & 0.70  & 0.67  & 0.68  & 0.61  & 0.67  & 0.67  & 0.67  & 0.66  & 0.70 \\
    PointNetLK~\cite{aoki2019pointnetlk} & 0.37  & 0.38  & 0.43  & 0.31  & 0.40  & 0.40  & 0.31  & 0.40  & 0.33  & 0.39  & 0.38  & 0.34 \\
    IT-Net~\cite{yuan2018iterative} & 0.27 &  0.19  & 0.10  & 0.22  & 0.17  & 0.20  & 0.28  & \textbf{0.31}  & 0.41  & 0.44  & 0.40  & 0.39 \\
    \midrule
    Ours & 0.27  & 0.17  & 0.11  & 0.20  & 0.16  & 0.18  & 0.19  & 0.43  & 0.40  & 0.48  & 0.33  & 0.31 \\
    Ours (siamese) & \textbf{0.20} & \textbf{0.09} & \textbf{0.08} & \textbf{0.10} & \textbf{0.08} & \textbf{0.11} & \textbf{0.08} & 0.40 & 0.35 & 0.34  & \textbf{0.32} & \textbf{0.30} \\
    \end{tabular}
    }
  \label{tab:rotation}%
\end{table}%
\paragraph{\textbf{Computational aspects.}}
As shown in~\cref{tab:classification} for ModelNet40 our network has $0.047M$ parameters. It incurs a computational cost in the order $O(MKL)$. The details are given in the supplementary material.
\paragraph{\textbf{Rotation estimation in 3D point clouds.}}
Our network can estimate both the absolute and relative 3D object rotations without pose-supervision. To evaluate this desired property, we used the well classified shapes on ModelNet10 dataset, a sub-dataset of Modelnet40~\cite{wu20153d}. This time, we use the official Modelenet10 dataset split with 3991 for training and 908 shapes for testing. 

During testing, we generate multiple instances per shape by transforming the instance with five arbitrary $SO(3)$ rotations. As we are also affected by the sampling of the point cloud, we resample the mesh five times and generate different pooling graphs across all the instances of the same shape. Our QE-architecture can estimate the pose in two ways: 1) \textit{canonical}: by directly using the output capsule with the highest activation, 2) \textit{siamese}: by a siamese architecture that computes the relative quaternion between the capsules that are maximally activated as shown in~\cref{fig:shapeAlign}. Both modes of operation are free of the data augmentation and we give further schematics of the latter in our appendix. 
It is worth mentioning that unlike regular pose estimation algorithms which utilize the same shape in both training and testing, our network never sees the test shapes during training. This is also known as \textit{category-level} pose estimation~\cite{wang2019normalized}.
\insertimageStar{1}{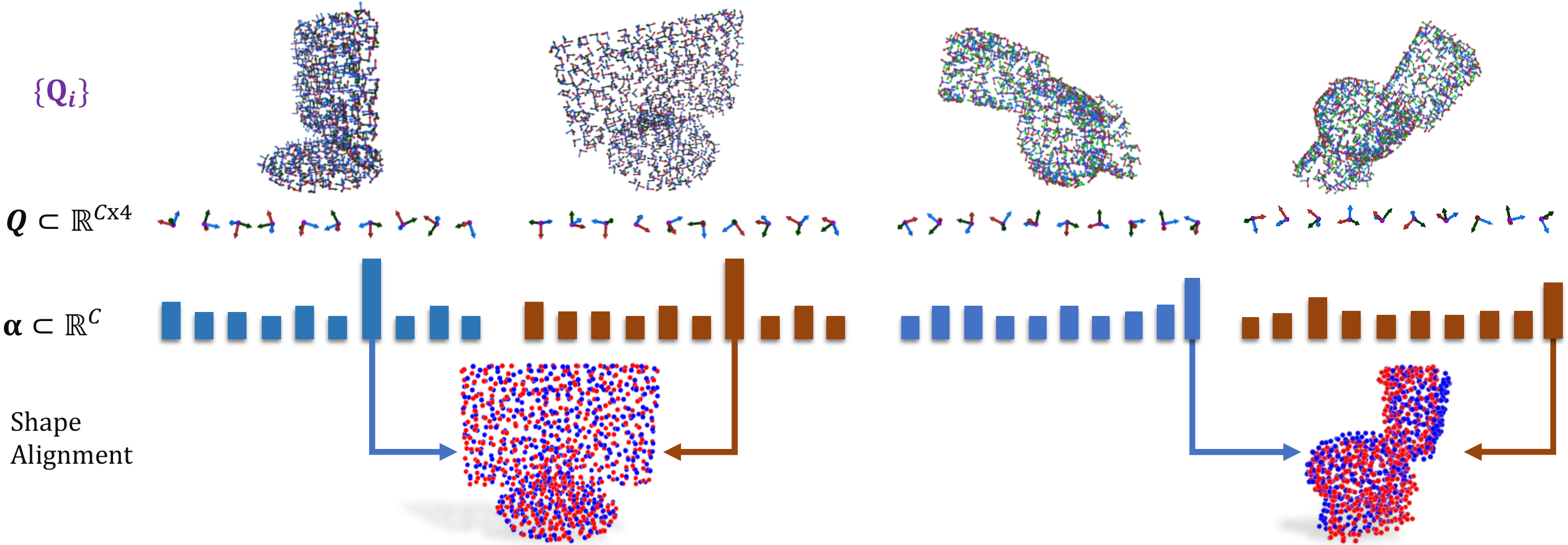}{Shape alignment on the \textbf{monitor} (left) and \textbf{toilet} (right) objects via our siamese equivariant capsule architecture. The shapes are assigned to the the maximally activated class. The corresponding pose capsule provides the rotation estimate.}{fig:shapeAlign}{t!}
 
Our results against the baselines including a naive averaging of the LRFs ({Mean LRF}) and principal axis alignment (PCA) are reported in~\cref{tab:rotation} as the relative angular error (RAE). We further include results of PointNetLK~\cite{aoki2019pointnetlk} and IT-Net~\cite{yuan2018iterative}, two state of the art 3D networks that iteratively align two given point sets. These methods are in nature similar to iterative closest point (ICP) algorithm~\cite{besl1992method} but 1) do not require an initialization (\eg first iteration estimates the pose), 2) learn data driven updates. 
Methods that use mesh inputs such as Spherical CNNs~\cite{esteves2018learning} cannot be included here as the random sampling of the same surface would not affect those. We also avoid methods that are just invariant to rotations (and hence cannot estimate the pose) such as Tensorfield Networks~\cite{thomas2018tensor}. Finally, note that , IT-net~\cite{yuan2018iterative} and PointNetLK~\cite{aoki2019pointnetlk} need to train for a lot of epochs (\eg~500) with random $SO(3)$ rotation augmentation in order to obtain models with full coverage of $SO(3)$, whereas we train only for $\sim 100$ epochs. Finally, the recent geometric capsule networks~\cite{srivastava2019geometric} remains similar to PCA with an RAE of $0.42$ on No\_Sym when evaluated under identical settings. We include more details about the baselines in the appendix. 


Relative Angle in Degrees (RAE) between the ground truth and the prediction is computed as: $d(\q_1,\q_2)/ \pi$.
Note that resampling and random rotations render the job of all methods difficult. However, both of our canonical and siamese versions which try to find a canonical and a relative alignment respectively, are better than the baselines.
As pose estimation of objects with rotational symmetry is a challenging task due to inherent ambiguities, we also report results on the non-symmetric subset (No\_Sym).
\paragraph{\textbf{Robustness against point and LRF resampling.}}
Density changes in the local neighborhoods of the shape are an important cause of error for our network. Hence, we ablate by applying random resampling (patch-wise dropout) objects in ModelNet10 dataset and repeating the classification and pose estimation as described above. While we use all the classes in classification accuracy, we only consider the well classified non-symmetric (No\_Sym) objects for ablating on the pose estimation.
The first part (LRF-10K) of~\cref{tab:ablation} shows our findings against gradual increases of the number of patches. Here, we sample 2K LRFs from the 10K LRFs computed on an input point set of cardinality 10K. 100\% dropout corresponds to 2K points in all columns. On second ablation, we reduce the amount of points on which we compute the LRFs, to 2K and 1K respectively. As we can see from the table, our network is robust towards the changes in the LRFs as well as the density of the points.


\begin{table}[t]
  \centering
  \caption{Ablation study on point density.}
  \setlength{\tabcolsep}{10.0pt}
  \resizebox{0.85\textwidth}{!}
  {
  \normalsize
    \begin{tabular}{l|cccc|c|c}
    LRF Input & \multicolumn{4}{c|}{LRF-10K}  & LRF-2K & LRF-1K \\
    \midrule
    Dropout & 50\%  & 66\%  & 75\%  & 100\% & 100\% & 100\% \\
    \midrule
    Classification Accuracy & 77.8  & 83.3  & 83.4  & 87.8  & 85.46 & 79.74 \\
    Angular Error & 0.34  & 0.27  & 0.25  & 0.09  & 0.10  & 0.12 \\
    \end{tabular}
    \label{tab:ablation}
    }
    \end{table}
\section{Conclusion and Discussion}
\label{sec:conclude}
We have presented a new framework for achieving permutation invariant and $SO(3)$ equivariant representations on 3D point clouds. Proposing a variant of the capsule networks, we operate on a sparse set of rotations specified by the input LRFs thereby circumventing the effort to cover the entire $SO(3)$. Our network natively consumes a compact representation of the group of 3D rotations - quaternions. We have theoretically shown its equivariance and established convergence results for our Weiszfeld dynamic routing by making connections to the literature of robust optimization. Our network by construction disentangles the object existence that is used as global features in classification. It
is among the few for having an explicit group-valued latent space and thus naturally estimates the orientation of the input shape, even without a supervision signal. 
\paragraph{\textbf{Limitations.}} In the current form our performance is severely affected by the shape symmetries. The length of the activation vector depends on the number of classes and for a sufficiently descriptive latent vector we need to have significant number of classes. On the other hand, this allows us to perform with merit on problems where the number of classes are large. The computation of LRFs are still sensitive to the point density changes and resampling. LRFs themselves can also be ambiguous and sometimes non-unique.
\paragraph{\textbf{Future work.}} Inspired by~\cite{cohen2019gauge} and~\cite{poulenard2018multi} our feature work will involve exploring the Lie algebra for equivariances, establishing invariance to the tangent directions, application of our network in the broader context of 6DoF object detection from point sets and looking for equivariances among point resampling.

\bibliographystyle{splncs04}


\begin{thebibliography}{10}
\providecommand{\url}[1]{\texttt{#1}}
\providecommand{\urlprefix}{URL }
\providecommand{\doi}[1]{https://doi.org/#1}

\bibitem{Afshar2018capsule}
{Afshar}, P., {Mohammadi}, A., {Plataniotis}, K.N.: Brain tumor type
  classification via capsule networks. In: 2018 25th IEEE International
  Conference on Image Processing (ICIP) (2018)

\bibitem{aftab2015convergence}
Aftab, K., Hartley, R.: Convergence of iteratively re-weighted least squares to
  robust m-estimators. In: Winter Conference on Applications of Computer
  Vision. IEEE (2015)

\bibitem{aftab2014generalized}
Aftab, K., Hartley, R., Trumpf, J.: Generalized weiszfeld algorithms for lq
  optimization. IEEE transactions on pattern analysis and machine intelligence
  \textbf{37}(4) (2014)

\bibitem{aftab2015}
Aftab, K., Hartley, R., Trumpf, J.: $l_q$ closest-point to affine subspaces
  using the generalized weiszfeld algorithm. International Journal of Computer
  Vision  (2015)

\bibitem{aoki2019pointnetlk}
Aoki, Y., Goforth, H., Srivatsan, R.A., Lucey, S.: Pointnetlk: Robust \&
  efficient point cloud registration using pointnet. In: Proceedings of the
  IEEE Conference on Computer Vision and Pattern Recognition. pp. 7163--7172
  (2019)

\bibitem{bao2019equivariant}
Bao, E., Song, L.: Equivariant neural networks and equivarification. arXiv
  preprint arXiv:1906.07172  (2019)

\bibitem{becigneul2018riemannian}
Becigneul, G., Ganea, O.E.: Riemannian adaptive optimization methods. In:
  International Conference on Learning Representations (2019)

\bibitem{besl1992method}
Besl, P.J., McKay, N.D.: Method for registration of 3-d shapes. In: Sensor
  fusion IV: control paradigms and data structures. vol.~1611, pp. 586--606.
  International Society for Optics and Photonics (1992)

\bibitem{birdal2020synchronizing}
Birdal, T., Arbel, M., Simsekli, U., Guibas, L.J.: Synchronizing probability
  measures on rotations via optimal transport. In: Proceedings of the IEEE/CVF
  Conference on Computer Vision and Pattern Recognition. pp. 1569--1579 (2020)

\bibitem{birdal2015point}
Birdal, T., Ilic, S.: Point pair features based object detection and pose
  estimation revisited. In: 2015 International Conference on 3D Vision. pp.
  527--535. IEEE (2015)

\bibitem{birdal2017point}
Birdal, T., Ilic, S.: A point sampling algorithm for 3d matching of irregular
  geometries. In: IEEE/RSJ International Conference on Intelligent Robots and
  Systems (IROS). IEEE (2017)

\bibitem{birdal2018bayesian}
Birdal, T., Simsekli, U., Eken, M.O., Ilic, S.: Bayesian pose graph
  optimization via bingham distributions and tempered geodesic mcmc. In:
  Advances in Neural Information Processing Systems. pp. 308--319 (2018)

\bibitem{Boomsma2017spherical}
Boomsma, W., Frellsen, J.: Spherical convolutions and their application in
  molecular modelling. In: Advances in Neural Information Processing Systems
  30. pp. 3433--3443 (2017)

\bibitem{burrus2012iterative}
Burrus, C.S.: Iterative reweighted least squares. OpenStax CNX. Available
  online: http://cnx. org/contents/92b90377-2b34-49e4-b26f-7fe572db78a1
  \textbf{12} (2012)

\bibitem{busam2016_iccvw}
Busam, B., Birdal, T., Navab, N.: Camera pose filtering with local regression
  geodesics on the riemannian manifold of dual quaternions. In: IEEE
  International Conference on Computer Vision Workshop (ICCVW) (October 2017)

\bibitem{chakraborty2018h}
Chakraborty, R., Banerjee, M., Vemuri, B.C.: H-cnns: Convolutional neural
  networks for riemannian homogeneous spaces. arXiv preprint arXiv:1805.05487
  (2018)

\bibitem{cohen2019gauge}
Cohen, T., Weiler, M., Kicanaoglu, B., Welling, M.: Gauge equivariant
  convolutional networks and the icosahedral {CNN}. In: Proceedings of the 36th
  International Conference on Machine Learning. pp. 1321--1330 (2019)

\bibitem{cohen2016group}
Cohen, T., Welling, M.: Group equivariant convolutional networks. In:
  International conference on machine learning. pp. 2990--2999 (2016)

\bibitem{cohen2018spherical}
Cohen, T.S., Geiger, M., K{\"{o}}hler, J., Welling, M.: Spherical cnns. In: 6th
  International Conference on Learning Representations, (ICLR) (2018)

\bibitem{cohen2018general}
Cohen, T.S., Geiger, M., Weiler, M.: A general theory of equivariant cnns on
  homogeneous spaces. In: Advances in Neural Information Processing Systems.
  pp. 9145--9156 (2019)

\bibitem{cohen2016steerable}
Cohen, T.S., Welling, M.: Steerable cnns. International Conference on Learning
  Representations (ICLR)  (2017)

\bibitem{CruzMota2012spherical}
Cruz-Mota, J., Bogdanova, I., Paquier, B., Bierlaire, M., Thiran, J.P.: Scale
  invariant feature transform on the sphere: Theory and applications.
  International Journal of Computer Vision  \textbf{98}(2),  217--241 (Jun
  2012)

\bibitem{deng2018ppf}
Deng, H., Birdal, T., Ilic, S.: Ppf-foldnet: Unsupervised learning of rotation
  invariant 3d local descriptors. In: European Conference on Computer Vision
  (ECCV) (2018)

\bibitem{deng2018ppfnet}
Deng, H., Birdal, T., Ilic, S.: Ppfnet: Global context aware local features for
  robust 3d point matching. In: Conference on Computer Vision and Pattern
  Recognition (2018)

\bibitem{esteves2018learning}
Esteves, C., Allen-Blanchette, C., Makadia, A., Daniilidis, K.: Learning so (3)
  equivariant representations with spherical cnns. In: Proceedings of the
  European Conference on Computer Vision (ECCV). pp. 52--68 (2018)

\bibitem{esteves2019cross}
Esteves, C., Sud, A., Luo, Z., Daniilidis, K., Makadia, A.: Cross-domain 3d
  equivariant image embeddings. In: International Conference on Machine
  Learning (ICML) (2019)

\bibitem{Esteves2019multiview}
Esteves, C., Xu, Y., Allen-Blanchette, C., Daniilidis, K.: Equivariant
  multi-view networks. In: Proceedings of the IEEE International Conference on
  Computer Vision. pp. 1568--1577 (2019)

\bibitem{Fey_2018_CVPR}
Fey, M., Eric~Lenssen, J., Weichert, F., Müller, H.: Splinecnn: Fast geometric
  deep learning with continuous b-spline kernels. In: The IEEE Conference on
  Computer Vision and Pattern Recognition (CVPR) (June 2018)

\bibitem{giles1987learning}
Giles, C.L., Maxwell, T.: Learning, invariance, and generalization in
  high-order neural networks. Applied optics  \textbf{26}(23),  4972--4978
  (1987)

\bibitem{hinton2011transforming}
Hinton, G.E., Krizhevsky, A., Wang, S.D.: Transforming auto-encoders. In:
  International Conference on Artificial Neural Networks. pp. 44--51. Springer
  (2011)

\bibitem{Hoppe1992}
Hoppe, H., DeRose, T., Duchamp, T., McDonald, J., Stuetzle, W.: Surface
  reconstruction from unorganized points, vol.~26.2. ACM (1992)

\bibitem{Jaiswal2019capsules}
Jaiswal, A., AbdAlmageed, W., Wu, Y., Natarajan, P.: Capsulegan: Generative
  adversarial capsule network. In: Computer Vision -- ECCV 2018 Workshops. pp.
  526--535. Springer International Publishing (2019)

\bibitem{jiang2018spherical}
Jiang, C.M., Huang, J., Kashinath, K., Prabhat, Marcus, P., Niessner, M.:
  Spherical {CNN}s on unstructured grids. In: International Conference on
  Learning Representations (2019)

\bibitem{khoury2017learning}
Khoury, M., Zhou, Q.Y., Koltun, V.: Learning compact geometric features. In:
  Proceedings of the IEEE International Conference on Computer Vision. pp.
  153--161 (2017)

\bibitem{kingma2014adam}
Kingma, D.P., Ba, J.: Adam: A method for stochastic optimization. arXiv
  preprint arXiv:1412.6980  (2014)

\bibitem{kondor2018clebsch}
Kondor, R., Lin, Z., Trivedi, S.: Clebsch--gordan nets: a fully fourier space
  spherical convolutional neural network. In: Advances in Neural Information
  Processing Systems (2018)

\bibitem{kondor2018generalization}
Kondor, R., Trivedi, S.: On the generalization of equivariance and convolution
  in neural networks to the action of compact groups. In: International
  Conference on Machine Learning. pp. 2747--2755 (2018)

\bibitem{kosiorek2019stacked}
Kosiorek, A., Sabour, S., Teh, Y.W., Hinton, G.E.: Stacked capsule
  autoencoders. In: Advances in Neural Information Processing Systems. pp.
  15512--15522 (2019)

\bibitem{laue2018computing}
Laue, S., Mitterreiter, M., Giesen, J.: Computing higher order derivatives of
  matrix and tensor expressions. In: Advances in Neural Information Processing
  Systems (2018)

\bibitem{lenssen2018group}
Lenssen, J.E., Fey, M., Libuschewski, P.: Group equivariant capsule networks.
  In: Advances in Neural Information Processing Systems. pp. 8844--8853 (2018)

\bibitem{li2018so}
Li, J., Chen, B.M., Hee~Lee, G.: So-net: Self-organizing network for point
  cloud analysis. In: Proceedings of the IEEE conference on computer vision and
  pattern recognition (2018)

\bibitem{li2018pointcnn}
Li, Y., Bu, R., Sun, M., Wu, W., Di, X., Chen, B.: Pointcnn: Convolution on
  x-transformed points. In: Advances in Neural Information Processing Systems
  (2018)

\bibitem{liao2019spherical}
Liao, S., Gavves, E., Snoek, C.G.: Spherical regression: Learning viewpoints,
  surface normals and 3d rotations on n-spheres. In: Proceedings of the IEEE
  Conference on Computer Vision and Pattern Recognition. pp. 9759--9767 (2019)

\bibitem{liu2018deep}
Liu, M., Yao, F., Choi, C., Ayan, S., Ramani, K.: Deep learning 3d shapes using
  alt-az anisotropic 2-sphere convolution. In: International Conference on
  Learning Representations (ICLR) (2019)

\bibitem{liu2019point2sequence}
Liu, X., Han, Z., Liu, Y.S., Zwicker, M.: Point2sequence: Learning the shape
  representation of 3d point clouds with an attention-based sequence to
  sequence network. In: Proceedings of the AAAI Conference on Artificial
  Intelligence. vol.~33, pp. 8778--8785 (2019)

\bibitem{magnus1985differentiating}
Magnus, J.R.: On differentiating eigenvalues and eigenvectors. Econometric
  Theory  \textbf{1}(2) (1985)

\bibitem{Marcos_2017_ICCV}
Marcos, D., Volpi, M., Komodakis, N., Tuia, D.: Rotation equivariant vector
  field networks. In: The IEEE International Conference on Computer Vision
  (ICCV) (Oct 2017)

\bibitem{markley2007averaging}
Markley, F.L., Cheng, Y., Crassidis, J.L., Oshman, Y.: Averaging quaternions.
  Journal of Guidance, Control, and Dynamics  \textbf{30}(4),  1193--1197
  (2007)

\bibitem{maturana2015voxnet}
Maturana, D., Scherer, S.: Voxnet: A 3d convolutional neural network for
  real-time object recognition. In: Intelligent Robots and Systems (IROS). IEEE
  (2015)

\bibitem{mehr2018manifold}
Mehr, E., Lieutier, A., Sanchez~Bermudez, F., Guitteny, V., Thome, N., Cord,
  M.: Manifold learning in quotient spaces. In: Proceedings of the IEEE
  Conference on Computer Vision and Pattern Recognition. pp. 9165--9174 (2018)

\bibitem{Melzi_2019_CVPR}
Melzi, S., Spezialetti, R., Tombari, F., Bronstein, M.M., Stefano, L.D.,
  Rodola, E.: Gframes: Gradient-based local reference frame for 3d shape
  matching. In: The IEEE Conference on Computer Vision and Pattern Recognition
  (CVPR) (June 2019)

\bibitem{petrelli2011repeatability}
Petrelli, A., Di~Stefano, L.: On the repeatability of the local reference frame
  for partial shape matching. In: 2011 International Conference on Computer
  Vision. IEEE (2011)

\bibitem{petrelli2012repeatable}
Petrelli, A., Di~Stefano, L.: A repeatable and efficient canonical reference
  for surface matching. In: 2012 Second International Conference on 3D Imaging,
  Modeling, Processing, Visualization \& Transmission. pp. 403--410. IEEE
  (2012)

\bibitem{poulenard2018multi}
Poulenard, A., Ovsjanikov, M.: Multi-directional geodesic neural networks via
  equivariant convolution. In: SIGGRAPH Asia 2018 Technical Papers. p.~236. ACM
  (2018)

\bibitem{qi2017pointnet}
Qi, C.R., Su, H., Mo, K., Guibas, L.J.: Pointnet: Deep learning on point sets
  for 3d classification and segmentation. In: Proceedings of the IEEE
  Conference on Computer Vision and Pattern Recognition. pp. 652--660 (2017)

\bibitem{qi2016volumetric}
Qi, C.R., Su, H., Nie{\ss}ner, M., Dai, A., Yan, M., Guibas, L.J.: Volumetric
  and multi-view cnns for object classification on 3d data. In: Proceedings of
  the IEEE conference on computer vision and pattern recognition. pp.
  5648--5656 (2016)

\bibitem{qi2017pointnet++}
Qi, C.R., Yi, L., Su, H., Guibas, L.J.: Pointnet++: Deep hierarchical feature
  learning on point sets in a metric space. In: Advances in neural information
  processing systems. pp. 5099--5108 (2017)

\bibitem{rezatofighi2017deepsetnet}
Rezatofighi, S.H., Milan, A., Abbasnejad, E., Dick, A., Reid, I., et~al.:
  Deepsetnet: Predicting sets with deep neural networks. In: 2017 IEEE
  International Conference on Computer Vision (ICCV). pp. 5257--5266. IEEE
  (2017)

\bibitem{sabour2018matrix}
Sabour, S., Frosst, N., Hinton, G.: Matrix capsules with em routing. In: 6th
  International Conference on Learning Representations, ICLR (2018)

\bibitem{sabour2017dynamic}
Sabour, S., Frosst, N., Hinton, G.E.: Dynamic routing between capsules. In:
  Advances in neural information processing systems. pp. 3856--3866 (2017)

\bibitem{Schutt2017Schnet}
Sch\"{u}tt, K., Kindermans, P.J., Sauceda~Felix, H.E., Chmiela, S., Tkatchenko,
  A., M\"{u}ller, K.R.: Schnet: A continuous-filter convolutional neural
  network for modeling quantum interactions. In: Advances in Neural Information
  Processing Systems (2017)

\bibitem{shen2018mining}
Shen, Y., Feng, C., Yang, Y., Tian, D.: Mining point cloud local structures by
  kernel correlation and graph pooling. In: Proceedings of the IEEE conference
  on computer vision and pattern recognition. pp. 4548--4557 (2018)

\bibitem{spezialetti2019learning}
Spezialetti, R., Salti, S., Stefano, L.D.: Learning an effective equivariant 3d
  descriptor without supervision. In: Proceedings of the IEEE International
  Conference on Computer Vision. pp. 6401--6410 (2019)

\bibitem{srivastava2019geometric}
Srivastava, N., Goh, H., Salakhutdinov, R.: Geometric capsule autoencoders for
  3d point clouds. arXiv preprint arXiv:1912.03310  (2019)

\bibitem{Steenrod1951}
Steenrod, N.E.: The topology of fibre bundles, vol.~14. Princeton University
  Press (1951)

\bibitem{thomas2018tensor}
Thomas, N., Smidt, T., Kearnes, S., Yang, L., Li, L., Kohlhoff, K., Riley, P.:
  Tensor field networks: Rotation-and translation-equivariant neural networks
  for 3d point clouds. arXiv preprint arXiv:1802.08219  (2018)

\bibitem{tombari2010unique}
Tombari, F., Salti, S., Di~Stefano, L.: Unique signatures of histograms for
  local surface description. In: European conference on computer vision. pp.
  356--369. Springer (2010)

\bibitem{wang2018an}
Wang, D., Liu, Q.: An optimization view on dynamic routing between capsules
  (2018), \url{https://openreview.net/forum?id=HJjtFYJDf}

\bibitem{wang2019normalized}
Wang, H., Sridhar, S., Huang, J., Valentin, J., Song, S., Guibas, L.J.:
  Normalized object coordinate space for category-level 6d object pose and size
  estimation. In: Proceedings of the IEEE Conference on Computer Vision and
  Pattern Recognition. pp. 2642--2651 (2019)

\bibitem{Wang_2018_CVPR}
Wang, S., Suo, S., Ma, W.C., Pokrovsky, A., Urtasun, R.: Deep parametric
  continuous convolutional neural networks. In: The IEEE Conference on Computer
  Vision and Pattern Recognition (CVPR) (June 2018)

\bibitem{dgcnn}
Wang, Y., Sun, Y., Liu, Z., Sarma, S.E., Bronstein, M.M., Solomon, J.M.:
  Dynamic graph cnn for learning on point clouds. ACM Transactions on Graphics
  (TOG)  (2019)

\bibitem{wang2019dynamic}
Wang, Y., Sun, Y., Liu, Z., Sarma, S.E., Bronstein, M.M., Solomon, J.M.:
  Dynamic graph cnn for learning on point clouds. ACM Transactions on Graphics
  (TOG)  \textbf{38}(5),  1--12 (2019)

\bibitem{weiler20183d}
Weiler, M., Geiger, M., Welling, M., Boomsma, W., Cohen, T.: 3d steerable cnns:
  Learning rotationally equivariant features in volumetric data. In: Advances
  in Neural Information Processing Systems. pp. 10381--10392 (2018)

\bibitem{Weiler_2018_CVPR}
Weiler, M., Hamprecht, F.A., Storath, M.: Learning steerable filters for
  rotation equivariant cnns. In: The IEEE Conference on Computer Vision and
  Pattern Recognition (CVPR) (June 2018)

\bibitem{Worrall_2018_ECCV}
Worrall, D., Brostow, G.: Cubenet: Equivariance to 3d rotation and translation.
  In: The European Conference on Computer Vision (ECCV) (September 2018)

\bibitem{Worrall_2017_CVPR}
Worrall, D.E., Garbin, S.J., Turmukhambetov, D., Brostow, G.J.: Harmonic
  networks: Deep translation and rotation equivariance. In: The IEEE Conference
  on Computer Vision and Pattern Recognition (CVPR) (July 2017)

\bibitem{wu20153d}
Wu, Z., Song, S., Khosla, A., Yu, F., Zhang, L., Tang, X., Xiao, J.: 3d
  shapenets: A deep representation for volumetric shapes. In: Proceedings of
  the IEEE conference on computer vision and pattern recognition. pp.
  1912--1920 (2015)

\bibitem{xinyi2018capsule}
Xinyi, Z., Chen, L.: Capsule graph neural network. In: International Conference
  on Learning Representations (ICLR) (2019),
  \url{openreview.net/forum?id=Byl8BnRcYm}

\bibitem{You2018prin}
You, Y., Lou, Y., Liu, Q., Tai, Y.W., Ma, L., Lu, C., Wang, W.: Pointwise
  rotation-invariant network with adaptive sampling and 3d spherical voxel
  convolution. In: AAAI. pp. 12717--12724 (2020)

\bibitem{yuan2018iterative}
Yuan, W., Held, D., Mertz, C., Hebert, M.: Iterative transformer network for 3d
  point cloud. arXiv preprint arXiv:1811.11209  (2018)

\bibitem{Zaheer2017}
Zaheer, M., Kottur, S., Ravanbakhsh, S., Poczos, B., Salakhutdinov, R.R.,
  Smola, A.J.: Deep sets. In: Advances in Neural Information Processing Systems
  (2017)

\bibitem{zhang2020quaternion}
Zhang, X., Qin, S., Xu, Y., Xu, H.: Quaternion product units for deep learning
  on 3d rotation groups. In: Proceedings of the IEEE/CVF Conference on Computer
  Vision and Pattern Recognition. pp. 7304--7313 (2020)

\bibitem{zhao20193d}
Zhao, Y., Birdal, T., Deng, H., Tombari, F.: 3d point capsule networks. In:
  Conference on Computer Vision and Pattern Recognition (CVPR) (2019)

\end{thebibliography}
\clearpage
\appendix
\section{Proof of Proposition 1}
\label{sec:proofProp1}
Before presenting the proof we recall the three individual statements contained in Prop. 1:
\begin{enumerate}
    \item $\Mean(\g \circ \Set, \w)$ is left-equivariant: $\Mean(\g \circ \Set, \w) = \g \circ \Mean(\Set, \w)$.
    \item Operator $\Mean$ is invariant under permutations: $\Mean(\{\q_{\sigma(1)},\dots,\q_{\sigma(Q)}\}, \w_\sigma)=\Mean(\{\q_1,\dots,\q_Q\}, \w)$.
    \item The transformations $\g \in \QH$ preserve the geodesic distance $\delta(\cdot)$.
\end{enumerate}
\begin{proof}
We will prove the propositions in order.
\begin{enumerate}
\item We start by transforming each element and replace $\q_i$ by $(\g \circ \q_i)$ of the cost defined in Eq. 4 of the main paper:
\begin{align}
    \q^\top\M\q &=\q^\top\Big( \sum\limits_{i=1}^{Q} w_i \q_i\q_i^\top\Big) \q \\
    &=\q^\top\Big( \sum\limits_{i=1}^{Q} w_i (\g \circ \q_i)(\g \circ \q_i)^\top\Big) \q \\
    &=\q^\top\Big( \sum\limits_{i=1}^{Q} w_i \G \q_i\q_i^\top \G^\top \Big) \q \\
    &=\q^\top\Big( \G\M_1\G^\top + \dots + \G\M_{Q}\G^\top \Big) \q \nonumber \\
    &=\q^\top \G\Big( \M_1\G^\top + \dots + \M_{Q}\G^\top \Big) \q \label{eq:addition}\\
    &=\q^\top \G\Big( \M_1 + \dots + \M_{Q} \Big)\G^\top \q \\
    &=\q^\top \G\M\G^\top \q\\
    &= \pt^\top \M \pt,
\end{align}
where $\M_i = w_i \q_i \q_i^\top$ and $\pt = \G^\top \q$. From orthogonallity of $\G$ it follows $\pt =\G^{-1}\q \implies \g \circ\pt =\q$ and hence $\g \circ \Mean(\Set, \w)=\Mean(\g \circ \Set, \w)$.
\item The proof follows trivially from the permutation invariance of the symmetric summation operator over the outer products in~\cref{eq:addition}.
\item It is sufficient to show that $|\q_1^\top \q_2|=|(\g\circ\q_1)^\top (\g\circ\q_2)|$ for any $\g\in\QH$:
\begin{align}
|(\g\circ\q_1)^\top (\g\circ\q_2)| &= |\q_1^\top \G^\top \G\q_2|  \\
&= |\q_1^\top \Id \q_2| \\
&= |\q_1^\top\q_2|,
\end{align}
where $\g\circ\q\equiv\G\q$. The result is a direct consequence of the orthonormality of $\G$.
\end{enumerate}
\end{proof}

\section{Proof of Lemma 1}
We will begin by recalling some preliminary definitions and results that aid us to construct the connection between the dynamic routing and the Weiszfeld algorithm.

\begin{dfn}[Affine Subspace]
A $d$-dimensional affine subspace of $R^N$ is obtained by a translation of a $d$-dimensional linear subspace $V\subset \R^N$ such that the origin is included in $S$:
\begin{align}
    S = \Big\{ \sum_{i=1}^{d+1} \alpha_i \x_i \,|\,  \sum_{i=1}^{d+1}\alpha_i = 1\Big\}.
\end{align}
Simplest choices for $S$ involve points, lines and planes of the Euclidean space.
\end{dfn}

\begin{dfn}[Orthogonal Projection onto an Affine Subspace]
An orthogonal projection of a point $\x\in\R^N$ onto an affine subspace explained by the pair $(\Ag, \Cg)$ is defined as:
\begin{align}
\Pi_i(\x)\triangleq\text{proj}_{S}(\x) = \Cg + \Ag(\x - \Cg).
\end{align}
$\Cg$ denotes the translation to make origin inclusive and $\Ag$ is a projection matrix typically defined via the orthonormal bases of the subspace.
\end{dfn}

\begin{dfn}[Distance to Affine Subspaces]
Distance from a given point $\x$ to a set of affine subspaces $\{S_1,S_2\dots S_k\}$ can be written as~\cite{aftab2015}:
\begin{align}
\label{eq:affineCost}
    C(\x) = \sum\limits_{i=1}^k d(\x, S_i) = \sum\limits_{i=1}^k \| \x - \text{proj}_{S_i}(\x) \|^2.
\end{align}
\end{dfn}

\begin{lemma}
Given that all the antipodal counterparts are mapped to the northern hemisphere, we will now think of the unit quaternion or \textit{versor} as the unit normal of a four dimensional hyperplane $h$, passing through the origin:
\begin{align}\label{eq:subspaceq}
h_i(\x) = \q_i^\top \x + q_d := 0.
\end{align}
$q_d$ is an added term to compensate for the shift. When $q_d=0$ the origin is incident to the hyperplane. With this perspective, quaternion $\q_i$ forms an affine subspace with $d=4$, for which the projection operator takes the form:
\begin{align}
\label{eq:affProj}
    \text{proj}_{S_i}(\pt) = (\Id - \q_i\q_i^\top)\pt
\end{align}
\end{lemma}
\begin{proof}
We consider~\cref{eq:affProj} for the case where $\Cg=\zero$ and $\Ag=(\Id - \q\q^\top)$. The former follows from the fact that our subspaces by construction pass through the origin. Thus, we only need to show that the matrix $\Ag=\Id-\q\q^\top$ is an orthogonal projection matrix onto the affine subspace spanned by $\q$. To this end, it is sufficient to validate that $\Ag$ is symmetric and idempotent: $\Ag^\top\Ag=\Ag\Ag=\Ag^2=\Ag$. Note that by construction $\q^\top\q$ is a symmetric matrix and hence $\Ag$ itself. Using this property and the unit-ness of the quaternion, we arrive at the proof:
\begin{align}
\Ag^\top\Ag&=(\Id-\q\q^\top)^\top(\Id-\q\q^\top)\\
&= (\Id-\q\q^\top)(\Id-\q\q^\top)\\
&= \Id-2\q\q^\top+\q\q^\top\q\q^\top\\
&= \Id-2\q\q^\top+\q\q^\top\\
&= \Id-\q\q^\top\triangleq \Ag
\end{align}
It is easy to verify that the projections are orthogonal to the quaternion that defines the subspace by showing $\text{proj}_S(\q)^\top\q = 0$:
\begin{align}
\q^\top\text{proj}_S(\q) = \q^\top\Ag\q 
=\q^\top (\Id - \q\q^\top)\q 
=\q^\top (\q - \q\q^\top\q)
=\q^\top (\q - \q) = 0.
\end{align}
Also note that this choice corresponds to $\text{tr}(\q\q^\top) = \sum_{i=1}^{d+1}\alpha_i = 1$.
\end{proof}

\begin{lemma}
The quaternion mean we suggest to use in the main paper~\cite{markley2007averaging} is equivalent to the Euclidean Weiszfeld mean on the affine quaternion subspaces.
\end{lemma}
\begin{proof}
We now recall and summarize the $L_q$-Weiszfeld Algorithm on affine subspaces~\cite{aftab2015}, which minimizes a $q$-norm variant of the cost defined in~\cref{eq:affineCost}:
\begin{align}
\label{eq:affineCostq}
    C_q(\x) = \sum\limits_{i=1}^k d(\x, S_i) = \sum\limits_{i=1}^k \| \x - \text{proj}_{S_i}(\x) \|^q.
\end{align}
Defining $\M_i=\Id-\Ag_i$, \cref{algo:LqWeiszfeld} summarizes the iterative procedure.
 \begin{algorithm2e} [h!]
 \DontPrintSemicolon
 \SetKwInOut{Input}{input}
 \Input{An initial guess $\x_0$ that does not lie any of the subspaces $\{S_i\}$, Projection operators $\Pi_i$, the norm parameter $q$}
 $\x^t \gets \x_0$\\
 \While{not converged}{
    Compute the weights $\w^t=\{w_i^t\}$:
    \begin{align}
        w_i^t = \| \M_i(\x^t-\Cg_i) \|^{q-2} \quad \forall i=1\dots k
    \end{align}\\
    Solve:
    \begin{align}
    \label{eq:weiszfeldUpdate}
    \x^{t+1} = \argmin_{\x \in \R^N} \sum\limits_{i=1}^k w_i^t \| \M_i(\x-\Cg_i) \|^2
    \end{align}\\
 }
 \caption{$L_q$ Weiszfeld Algorithm on Affine Subspaces~\cite{aftab2015}.}
 \label{algo:LqWeiszfeld}
 \end{algorithm2e}
 
Note that when $q=2$, the algorithm reduces to the computation of a non-weighted mean ($w_i=1\,\forall i$), and a closed form solution exists for~\cref{eq:weiszfeldUpdate} and is given by the normal equations:
 \begin{align}
        \x = \Big( \sum\limits_{i=1}^k w_i\M_i \Big)^{-1} \Big( \sum\limits_{i=1}^k w_i\M_i\Cg_i \Big)
\end{align}

For the case of our quaternionic subspaces $\Cg=\zero$ and we seek the solution that satisfies:
\begin{align}
   \Big( \sum\limits_{i=1}^k \M_i \Big)\x =\Big( \frac{1}{k}\sum\limits_{i=1}^k \M_i \Big)\x = \zero.
\end{align}
It is well known that the solution to this equation under the constraint $\|\x\|=1$ lies in nullspace of $\M=\frac{1}{k}\sum\limits_{i=1}^k \M_i$ and can be obtained by taking the singular vector of $\M$ that corresponds to the largest singular value. Since $\M_i$ is idempotent, the same result can also be obtained through the eigendecomposition:
\begin{align}
    \q^\star = \argmax_{\q\in \Scal^3} \q\M\q
\end{align}
which gives us the unweighted Quaternion mean~\cite{markley2007averaging}.
\end{proof}

\section{Proof of Theorem 1}
Once the Lemma 1 is proven, we only need to apply the direct convergence results from the literature.
Consider a set of points $\Y=\{\y_1\dots \y_K\}$ where $K>2$ and $\y_i\in\QH$. Due to the compactness, we can speak of a ball $\B(\rvo, \rho)$ encapsulating all $\y_i$. We also define the $\D=\{\x\in \QH \,|\, C_q(\x)<C_q(\rvo) \}$, the region where the loss decreases. 

We first state the assumptions that permit our theoretical result. These assumptions are required by the works that establish the convergence of such Weiszfeld algorithms~\cite{aftab2015convergence,aftab2014generalized} :\vspace{3mm}
\newline\textbf{\quad H1.} $\y_1\dots \y_K$ should not lie on a single geodesic of the quaternion manifold. 
\newline\textbf{\quad H2.} $\D$ is bounded and compact. The topological structure of $SO(3)$ imposes a bounded convexity radius of $\rho<\pi / 2$. 
\newline\textbf{\quad H3.} The minimizer in~\cref{eq:weiszfeldUpdate} is continuous.
\newline\textbf{\quad H4.} The weighting function $\sigma(\cdot)$ is concave and differentiable.
\newline\textbf{\quad H5.} Initial quaternion (in our network chosen randomly) does not belong to any of the subspaces.

Note that $\textbf{H5}$ is not a strict requirement as there are multiple ways to circumvent (simplest being a re-initialization). Under these assumptions, the sequence produced by~\cref{eq:weiszfeldUpdate} will converge to a critical point unless $\x^t=\y_i$ for any $t$ and $i$~\cite{aftab2014generalized}. For $q=1$, this critical point is on one of the subspaces specified in~\cref{eq:subspaceq} and thus is a geometric median.\qed

Note that due to the assumption $\textbf{H2}$, we cannot converge from any given point. For randomly initialized networks this is indeed a problem and does not guarantee practical convergence. Yet, in our experiments we have not observed any issue with the convergence of our dynamic routing. As our result is one of the few ones related to the analysis of DR, we still find this to be an important first step.

For different choices of $q:1\leq q \leq 2$, the weights take different forms. In fact, this IRLS type of algorithm is shown to converge for a larger class of weighting choices as long as the aforementioned conditions are met. That is why in practice we use a simple sigmoid function.

\insertimageStar{1}{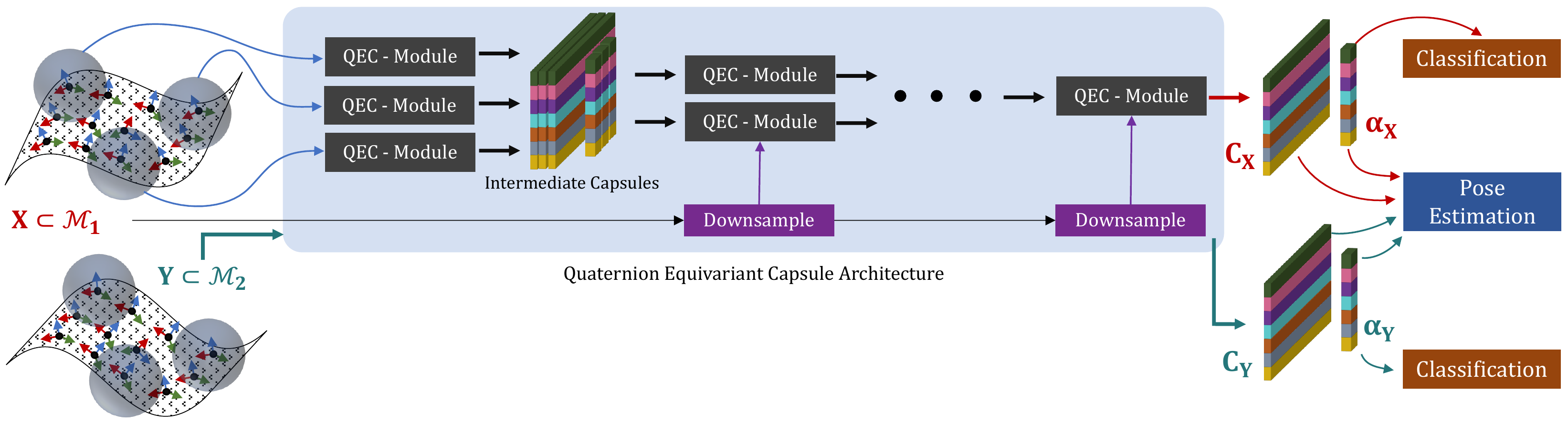}{Our siamese architecture used in the estimation of relative poses. We use a shared network to process two distinct point clouds $(\X,\Y)$ to arrive at the latent representations $(\rmC_X,\bm{\alpha}_X)$ and $(\rmC_Y,\bm{\alpha}_Y)$ respectively. We then look for the highest activated capsules in both point sets and compute the rotation from the corresponding capsules. Thanks to the rotations disentangled into capsules, this final step simplifies to a relative quaternion calculation.}{fig:qsiamese}{t}

\section{Further Discussions}
\paragraph{\textbf{On convergence, runtime and complexity.}}
Note that  while the convergence basin is known, to the best of our knowledge, a convergence rate for a Weiszfeld algorithm in affine subspaces is not established. From the literature of robust minimization via Riemannian gradient descent (this is essentially the corresponding particle optimizer), we conjecture that such a rate depends upon the choice of the convex regime (in this case $1\leq q \leq 2$) and is at best linear -- though we did not prove this conjecture. In practice we run the Weiszfeld iteration only 3 times, similar to the original dynamic routing. This is at least sufficient to converge to a point good enough for the network to explain the data at hand.

QEC module summarized in the Alg. 2 of the main paper can be dissected into three main steps: (i) canonicalization of the local oriented point set, (ii) the $t$-kernel and (iii) dynamic routing. Overall the total computational complexity reads $O(L + K C_{MLP} + C_{DR})$ where $C_{MLP}$ and $C_{DR}$ are the computational costs of the MLP and the DR respectively:
\begin{align}\label{eq:complex}
    C_{DR} &= LM + M(K + k (2L) + L) = M( K+2(k+1)L )\nonumber \\
    C_{MLP} &= 64N_c + 4 M N_c^2.
\end{align}
Note that~\cref{eq:complex} depicts the complexity of a single QEC module. In our architecture we use a stack of those each of which cause an added increase in the complexity proportional to the number of points downsampled.

Our weighted quaternion average relies upon a differentiable SVD.  While not increasing the theoretical computational complexity, when done naively, this operation can cause significant increase in runtime. Hence, we compute the SVD using CUDA kernels in a batch-wise manner. This batch-wise SVD makes it possible to average a large amount of quaternions with high efficiency. Note that we omit the computational aspects of LRF calculation as we consider it to be an input to our system and different LRFs exhibit different costs.

We have further conducted a runtime analysis in the \emph{3D Shape Classification} experiment on an Nvidia GeForce RTX 2080 Ti with the network configuration mentioned in Sec. 5.2 of the main paper. During training, each batch (where batch size $b = 8$) takes $0.226$s and $1939$M of GPU memory. During inference, processing each instance takes $0.036$s and consumes $1107$M of GPU memory.  
\begin{figure}[t]
    \centering
    \subfigure[]{\includegraphics[width=0.47\textwidth]{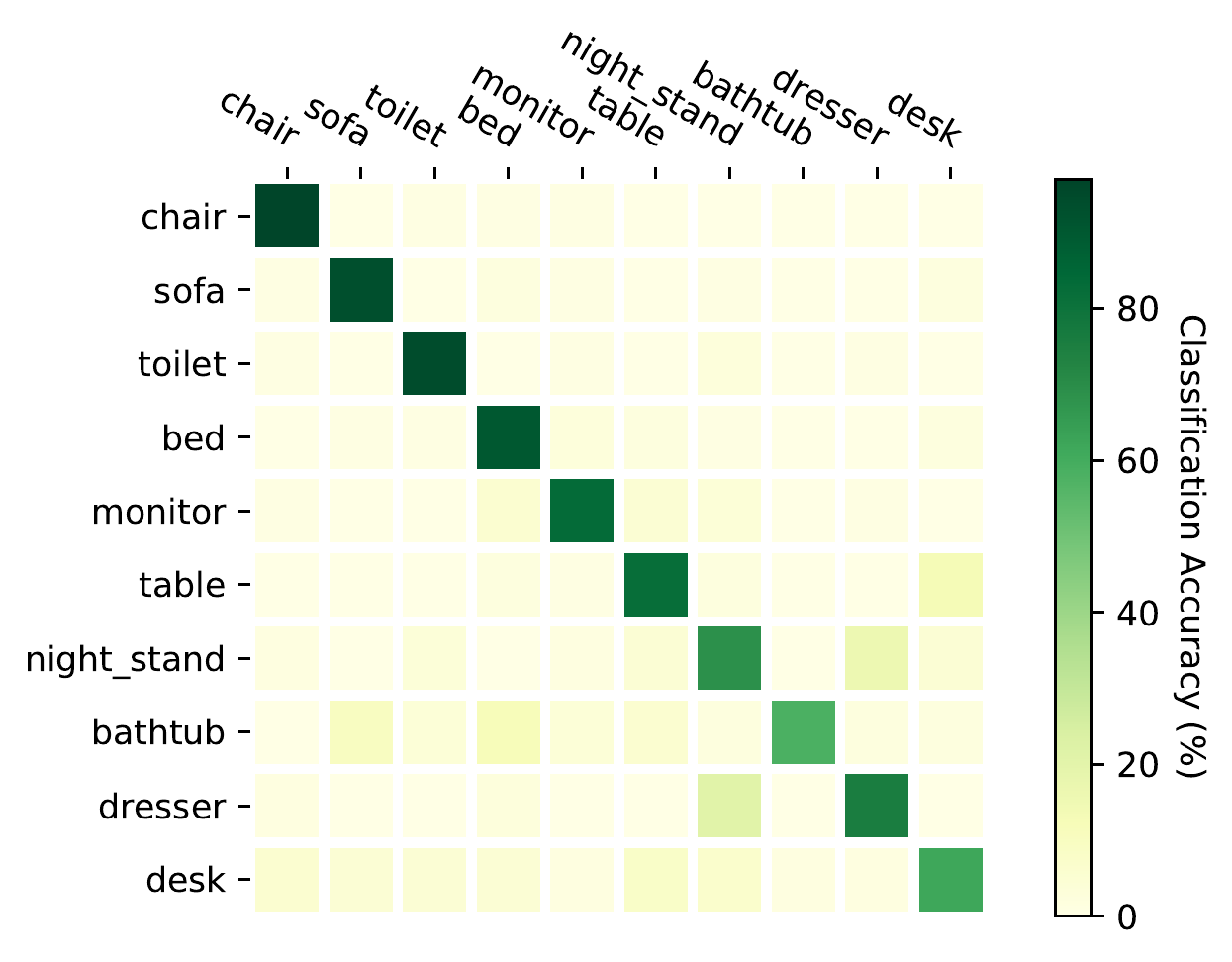}
    \label{fig:confusion}
    }\hfill
    \subfigure[]{\includegraphics[width=0.48\textwidth]{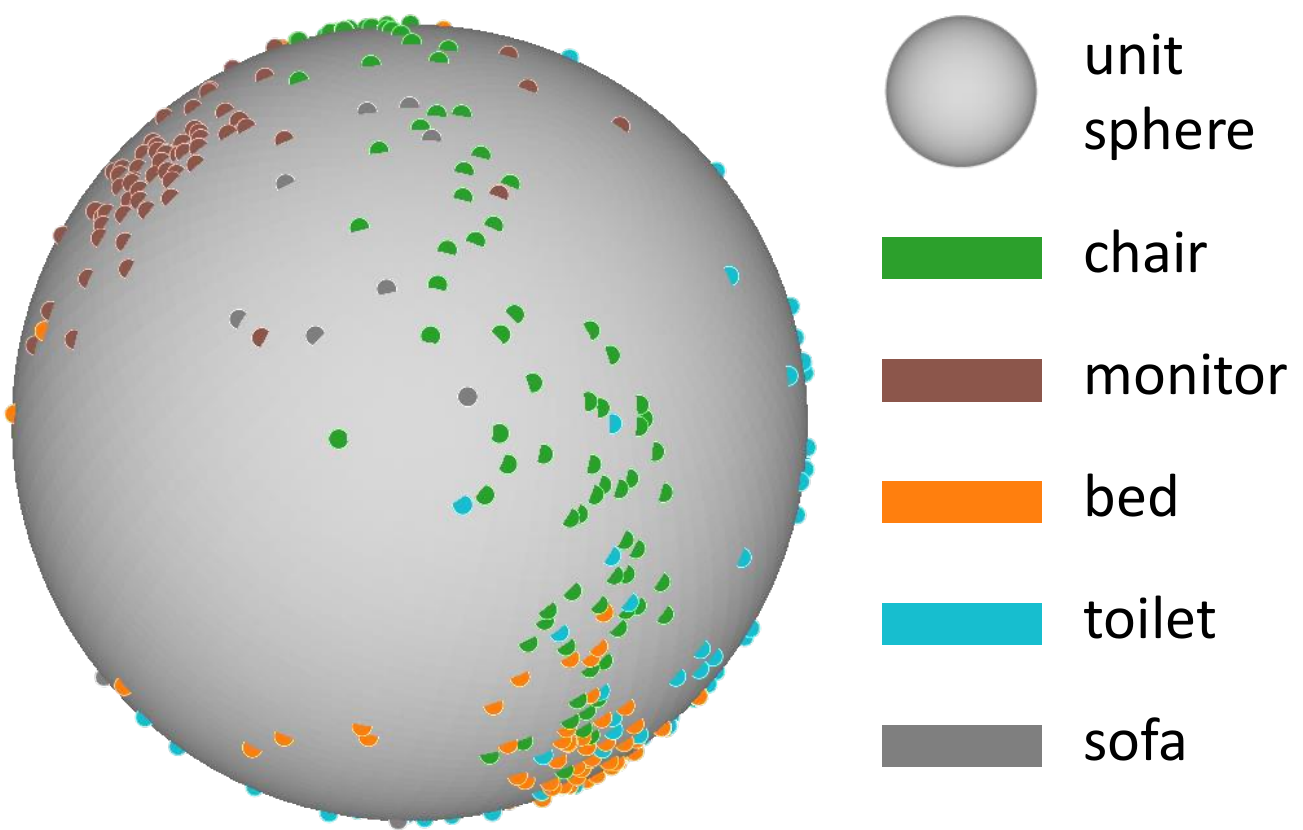}
    \label{fig:modelnetsphere}
    }
    \caption{\textbf{(a) } Confusion matrix on ModelNet10 for classification. \textbf{(b) } Distribution of initial poses per class.}
\end{figure}

Note that the use of LRFs helps us to restrict the rotation group to certain elements and thus we can use networks with significantly less parameters (as low as $0.44M$) compared to others as shown in Tab. 1 of the main paper. Number of parameters in our network depends upon the number of classes, e.g. for ModelNet10 we have $0.047M$ parameters.

\paragraph{\textbf{Quaternion ambiguity.}}
Quaternions of the northern and southern hemispheres represent the same exact rotation, hence one of them is \emph{redundant}. By mapping one hemisphere to the other, we sacrifice the closeness of the manifold. This could slightly distort the behavior of the linearization operator around the Ecuador. However, the rest of the operations such as geodesic distances respect such antipodality, as we consider the Quaternionic manifold and not the sphere. When the subset of operations we develop and the nature of local reference frames are concerned, we did not find this transformation to cause serious shortcomings.

\paragraph{\textbf{Performance on different shapes with same orientation.}}
The NR/NR scenario in Tab. 1 of the main paper involves classification on different shapes within a category without rotation, \eg chairs with different shapes. In addendum, we now provide in~\cref{fig:modelnetsphere} an additional insight into the pose distribution for all canonicalized objects within a class. To do so, we rotate the horizontal standard basis vector $\mathbf{e}_x = [1,0,0]$ using the predict quaternion (the most activated output capsule) and plot the resulting point on a unit sphere as shown in~\cref{fig:modelnetsphere}. A qualitative observation reveals that for all five non-symmetric classes, the poses of all the instances within a class would form a cluster. This roughly holds across all classes and indicates that the relative pose information is consistent within the classes. On the other hand, objects with symmetries form multiple clusters. 

\section{Our Siamese Architecture}
For estimation of the relative pose with supervision, we benefit from a Siamese variation of our network. In this case, latent capsule representations of two point sets $\X$ and $\Y$ jointly contribute to the pose regression as shown in~\cref{fig:qsiamese}.

We show additional results from the computation of local reference frames and the multi-channel capsules deduced from our network in~\cref{fig:LRFsAppendix}.

\insertimageStar{1}{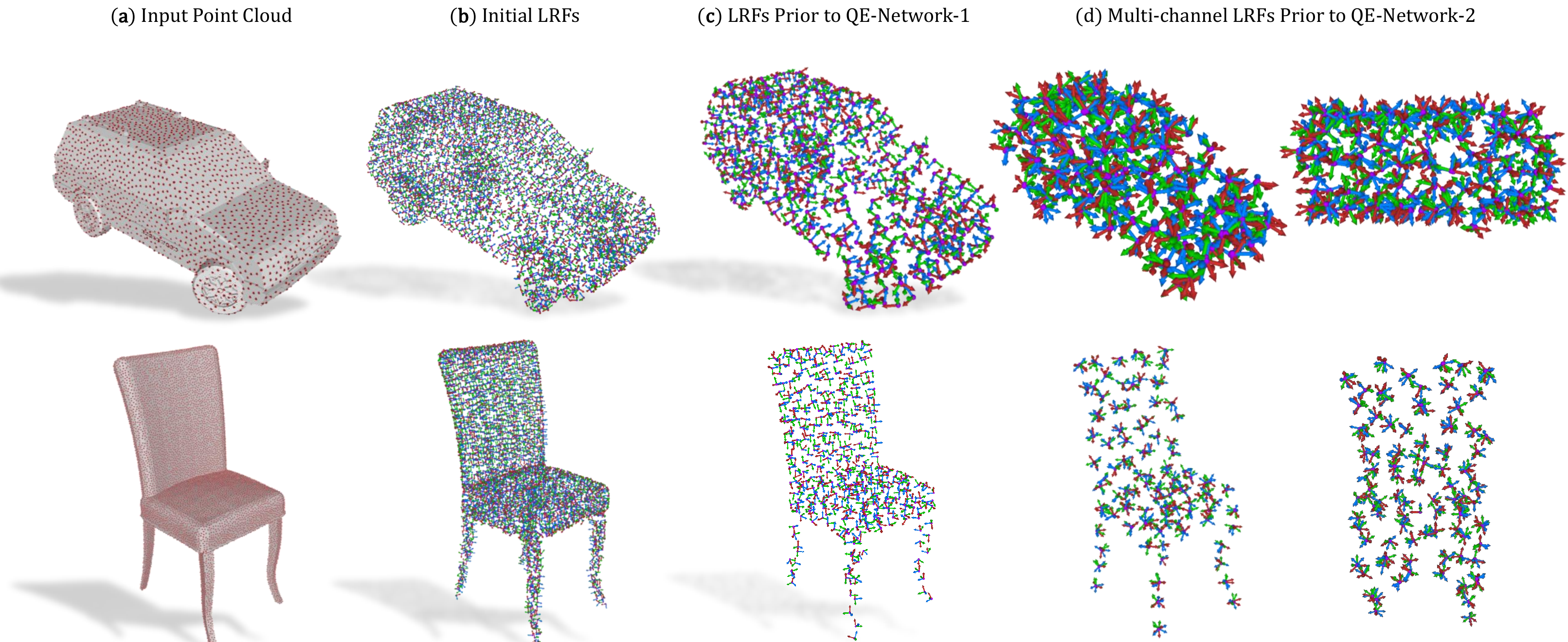}{Additional intermediate results on car (first row) and chair (second row) objects. This figure supplements Fig. 1(a) of the main paper.}{fig:LRFsAppendix}{t}

\section{Additional Details on Evaluations}
\paragraph{\textbf{Details on the evaluation protocol.}} For Modelnet40 dataset used in Tab. 1, we stick to the official split with 9,843 shapes for training and 2,468 different shapes for testing. For rotation estimation in Tab. 2, we again used the official Modelenet10 dataset split with 3991 for training and 908 shapes for testing. 3D point clouds (10K points) are randomly sampled from the mesh surfaces of each shape~\cite{qi2017pointnet,qi2017pointnet++}.
The objects in training and testing dataset are different, but they are from the same categories so that they can be oriented meaningfully.  
During training, we did not augment the dataset with random rotations. All the shapes are trained with single orientation (well-aligned). We call this \textit{trained with NR}. During testing, we randomly generate multiple arbitrary $SO(3)$ rotations for each shape and evaluate the average performance for all the rotations. This is called \textit{test with AR}. This protocol is used in both our algorithms and the baselines.

\insertimageStar{1}{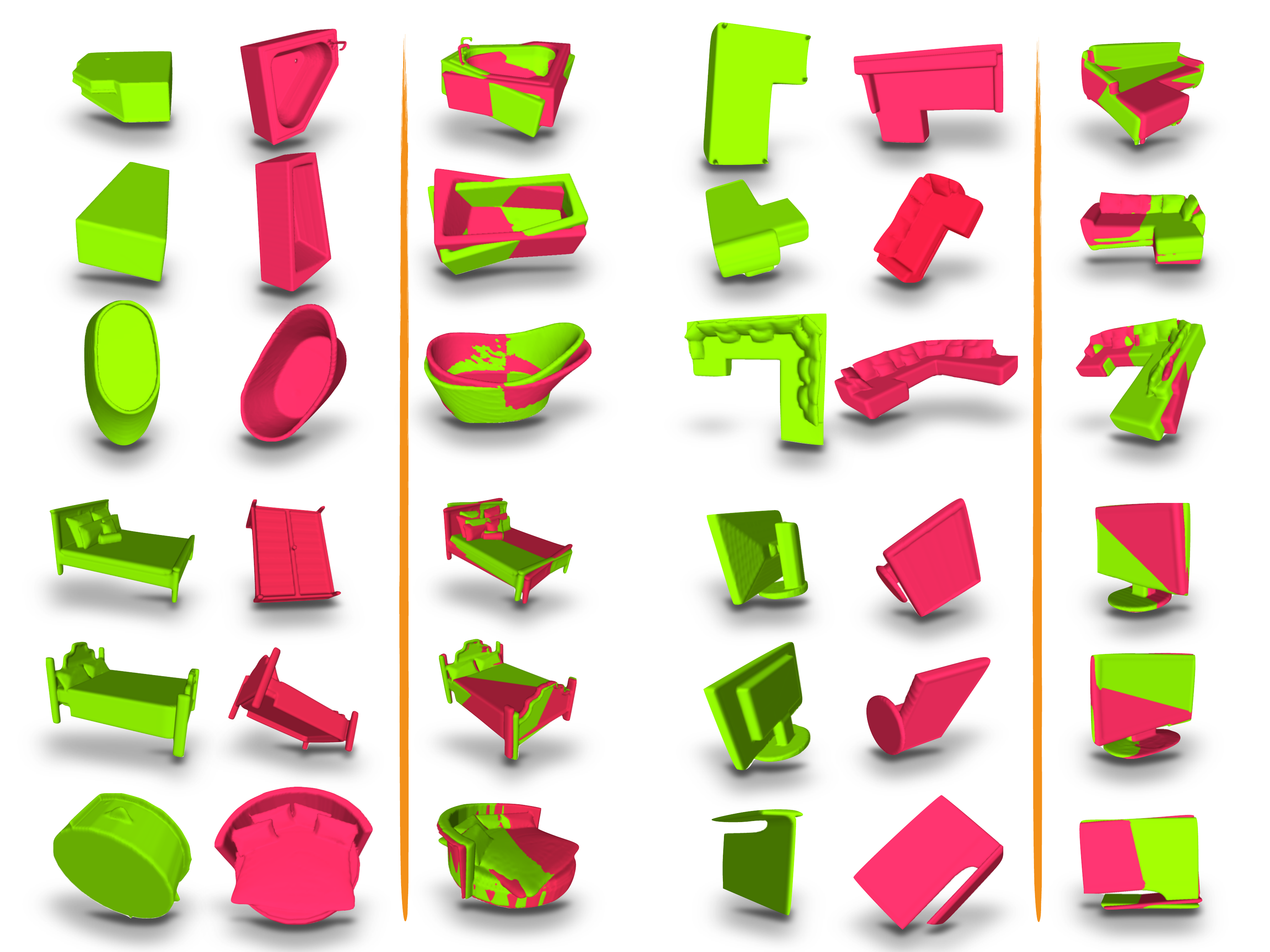}{Additional pairwise shape alignment on more categories in Modelnet10 dataset. We do not perform any ICP and the transformations that align the two point clouds are direct results of the forward pass of our Siamese network.}{fig:ShapeAlignAppendix}{t}


\paragraph{\textbf{Confusion of classification in ModelNet.}}
To provide additional insight into how our activation features perform, we now report the confusion matrix in the task of classification on the all the objects of ModelNet10. Unique to our algorithm, the classification and rotation estimation reinforces one another. As seen from~\cref{fig:confusion} on the right, the first five categories that exhibit less rotational symmetry has the higher classification accuracy than their rotationally symmetric counterparts. 

\paragraph{\textbf{Distribution of errors reported in Tab. 2.}}
We now provide more details on the errors attained by our algorithm as well as the state of the art. To this end, we report, in~\cref{fig:PoseEvalBars} the histogram of errors that fall within quantized ranges of orientation errors. It is noticeable that our Siamese architecture behaves best in terms of estimating the objects rotation. For completeness, we also included the results of the variants presented in our ablation studies: Ours-2kLRF, Ours-1kLRF. They evaluate the model on the re-calculated LRFs in order to show the robustness towards to various point densities. We have also modified IT-Net and PointNetLK only to predict rotation because the original works predict both rotations and translations. Finally, note here that we do not use data augmentation for training our networks (see AR), while both for PointNetLK and for IT-Net we do use augmentation.

\insertimageStar{0.95}{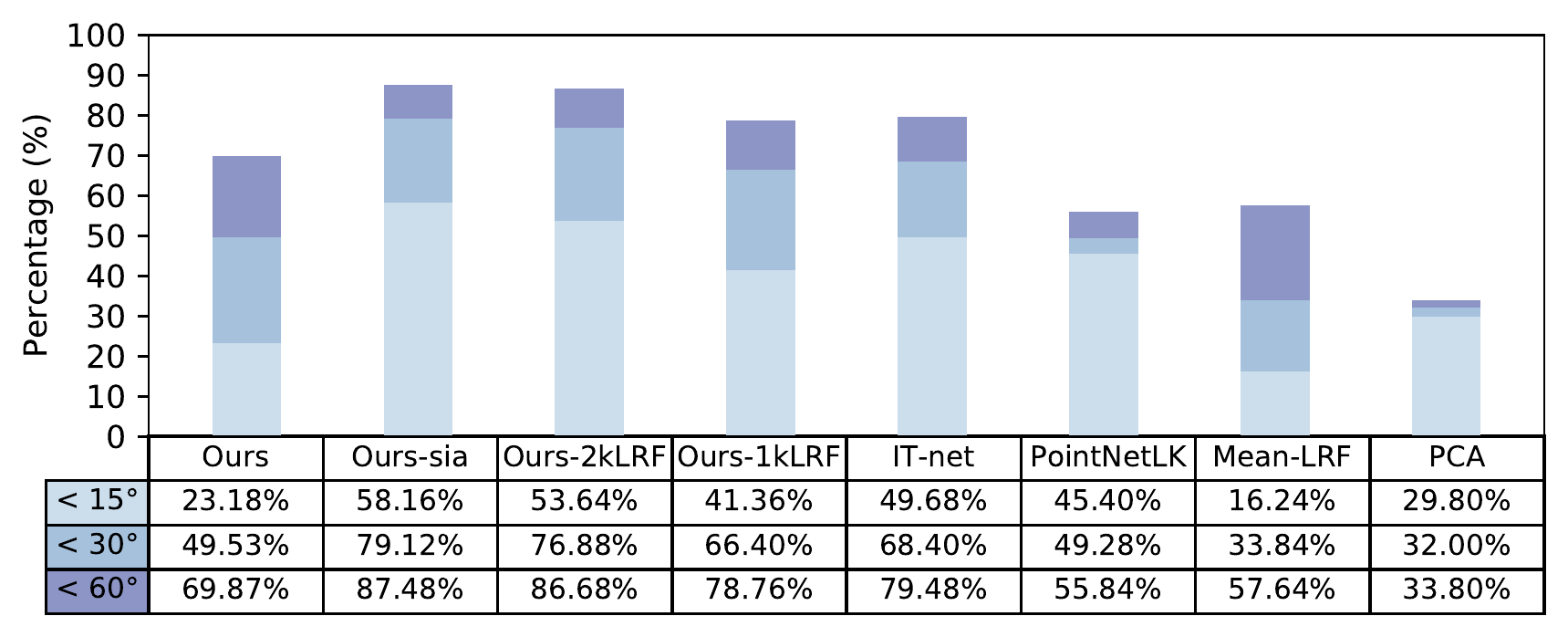}{Cumulative error histograms of rotation estimation on ModelNet10. Each row ($<\theta^\circ$) of this extended table shows the percentage of shapes that have rotation error less than $\theta$. The colors of the bars correspond to the rows they reside in. The higher the errors are contained in the first bins (light blue) the better. Vice versa, the more the errors are clustered toward the $60^\circ$ the worse the performance of the method.}{fig:PoseEvalBars}{ht}



\end{document}